\documentclass{article}

\usepackage{arxiv}

\usepackage[utf8]{inputenc} % allow utf-8 input
\usepackage[T1]{fontenc}    % use 8-bit T1 fonts
\usepackage{hyperref}       % hyperlinks
\usepackage{url}            % simple URL typesetting
\usepackage{booktabs}       % professional-quality tables
\usepackage{amsfonts}       % blackboard math symbols
\usepackage{nicefrac}       % compact symbols for 1/2, etc.
\usepackage{microtype}      % microtypography
\usepackage{lipsum}		% Can be removed after putting your text content
\usepackage{graphicx}
\usepackage{natbib}
\usepackage{doi}

\usepackage{amsmath,amsfonts,bm}

\def\eqref#1{equation~\ref{#1}}
\def\1{\bm{1}}

\DeclareMathAlphabet{\mathsfit}{\encodingdefault}{\sfdefault}{m}{sl}
\SetMathAlphabet{\mathsfit}{bold}{\encodingdefault}{\sfdefault}{bx}{n}

\newcommand{\R}{\mathbb{R}}

\usepackage{hyperref}
\usepackage{url}

\usepackage{color}
\usepackage{comment}
\usepackage[utf8]{inputenc} % allow utf-8 input
\usepackage[T1]{fontenc}    % use 8-bit T1 fonts

\usepackage{booktabs}       % professional-quality tables
\usepackage{amsfonts}       % blackboard math symbols
\usepackage{nicefrac}       % compact symbols for 1/2, etc.
\usepackage{microtype}      % microtypography
\usepackage{physics}
\usepackage{yfonts}
\usepackage{setspace}
\usepackage{amsmath,amsfonts,amssymb,amsthm,epsfig,epstopdf,url,array,booktabs}
\usepackage{algcompatible}
\usepackage[ruled,vlined]{algorithm2e}
\usepackage[font=small,labelfont=bf]{caption}
\newcommand\numberthis{\addtocounter{equation}{1}\tag{\theequation}}
\theoremstyle{plain}
\newtheorem{thm}{Theorem}[section]
\newtheorem{lem}[thm]{Lemma}

\usepackage{xspace}
\usepackage{caption}
\usepackage{subcaption}
\theoremstyle{definition}
\newtheorem{defn}{Definition}[section]

\usepackage{csquotes}
\usepackage{tcolorbox}

\theoremstyle{remark}

\usepackage{multirow}

\newcommand{\squishlist}{
	\begin{list}{$\bullet$}
		{
			\setlength{\itemsep}{0pt}
			\setlength{\parsep}{3pt}
			\setlength{\topsep}{3pt}
			\setlength{\partopsep}{0pt}
			\setlength{\leftmargin}{1em}
			\setlength{\labelwidth}{0.5em}
			\setlength{\labelsep}{0.5em} } }
	
\newcommand{\squishend}{
\end{list}  }
\newcommand{\squishlistnum}{
	\begin{enumerate}
		{
			\setlength{\itemsep}{0pt}
			\setlength{\parsep}{3pt}
			\setlength{\topsep}{3pt}
			\setlength{\partopsep}{0pt}
			\setlength{\leftmargin}{1.5em}
			\setlength{\labelwidth}{1em}
			\setlength{\labelsep}{0.5em} } }
	
\newcommand{\squishendnum}{
\end{enumerate}  }

\definecolor{caseycolor}{RGB}{78,90,255}

\usepackage{blindtext}

\usepackage{wrapfig}
\usepackage{xcolor}

\usepackage[utf8]{inputenc} % allow utf-8 input
\usepackage[T1]{fontenc}    % use 8-bit T1 fonts
\usepackage{hyperref}       % hyperlinks
\usepackage{url}            % simple URL typesetting
\usepackage{booktabs}       % professional-quality tables
\usepackage{amsfonts}       % blackboard math symbols
\usepackage{nicefrac}       % compact symbols for 1/2, etc.
\usepackage{lipsum}

\newcommand{\bx}{\mathbf{x}}
\newcommand{\by}{\mathbf{y}}

\newcommand{\bz}{\mathbf{z}}

\newcommand{\bt}{\mathbf{t}}

\newcommand{\calX}{\mathcal{X}}
\newcommand{\calT}{\mathcal{T}}

\newcommand{\calP}{\mathcal{P}}
\newcommand{\calA}{\mathcal{A}}
\newcommand{\calG}{\mathcal{G}}

\newcommand{\calY}{\mathcal{Y}}

\theoremstyle{plain}
\newtheorem{lemma}{Lemma} 
\newtheorem{prope}{Property}

\theoremstyle{definition}
\newcommand{\DP}{\textsf{DP}~}
\newcommand{\ldp}{\textsf{LDP}~}
\newcommand{\DO}{\textsf{DO}}
\newcommand{\name}{$d_\sigma$}

\title{Privacy Implications of Shuffling}

\author{%
  Casey Meehan \\
  UC San Diego\\
  \texttt{cmeehan@eng.ucsd.edu} \\
   \And
   Amrita Roy Chowdhury \\
   University of Wisconsin-Madison \\
   \texttt{roychowdhur2@wisc.edu} \\
   \AND
   Kamalika Chaudhuri \\
   UC San Diego\\
   \texttt{kamalika@eng.ucsd.edu} \\
   \And 
   Somesh Jha \\
   University of Wisconsin-Madison \\
   jha@cs.wisc.edu
}

\renewcommand{\shorttitle}{\textit{arXiv} Template}

\hypersetup{
pdftitle={A template for the arxiv style},
pdfsubject={q-bio.NC, q-bio.QM},
pdfauthor={David S.~Hippocampus, Elias D.~Striatum},
pdfkeywords={First keyword, Second keyword, More},
}

\begin{document}
\maketitle

\begin{abstract}
\ldp deployments are vulnerable to inference attacks as an adversary can link the noisy responses to their identity and subsequently, auxiliary information using the \textit{order} of the data. An alternative model, shuffle \textsf{DP}, prevents this by shuffling the noisy responses uniformly at random.  However, this limits the data learnability -- only symmetric functions (input order agnostic) can be learned. In this paper, we strike a balance and show that systematic shuffling of the noisy responses can thwart specific inference attacks while retaining some meaningful data learnability. To this end, we propose a novel privacy guarantee, \name-privacy, that captures the privacy of the order of a data sequence. \name-privacy allows tuning the granularity at which the ordinal information is maintained, which formalizes the degree the resistance to inference attacks trading it off with data learnability.  Additionally, we propose a novel shuffling mechanism that can achieve \name-privacy and demonstrate the practicality of our mechanism via evaluation on real-world datasets. 
\end{abstract}

% keywords can be removed
% \keywords{Data privacy \and Second keyword \and More}

% \vspace{-0.4cm} 
\section{Introduction}
\label{sec:intro}\vspace{-0.25cm}
Differential Privacy (\textsf{DP}) and its local variant (\textsf{LDP}) are the most commonly accepted notions of data privacy. \ldp has the significant advantage of not requiring a trusted centralized aggregator, and has become a popular model for commercial deployments, such as those of Microsoft~\citep{Microsoft}, Apple~\citep{Apple}, and Google~\citep{Rappor1,Rappor2,Prochlo}. Its formal guarantee asserts that an adversary cannot infer the value of an individual's private input by observing the noisy output. However in practice, a vast amount of \textit{public auxiliary information}, such as address, social media connections, %(in terms of friends/followers), 
court records, property records, %\citep{home},
income and birth dates \citep{birth}, is available for every individual. An adversary, with access to such auxiliary information, \emph{can} learn about an individual's private data from several \emph{other} participants' noisy responses. We illustrate this as follows.%\vspace{-0.1cm}Its formal guarantee asserts that an adversary does not learn much about a participant's private input after observing the mechanism's noisy representation of it. However, an adversary \emph{can} draw reliable inferences about a participant's private input after observing several \emph{other} participants' noisy responses. Practically speaking, this is due to the fact that \ldp does not formally hide the identity of data owners: e.g. an \ldp response from a participant's browser can still be linked to that participant's device or account and thereby to their identity. With the preponderance of public auxiliary information available today (e.g. address, social media connections, court records, property records \nocite{home}, income and birth rates \citep{birth}), an adversary can track down other participants who (positively or negatively) correlate with them and leverage those responses to make a reliable inference of their underlying private input. We illustrate this as follows: 

% However in practice, a vast amount of \textit{public auxiliary information}, such as address, social media connections, %(in terms of friends/followers), 
% court records, property records \cite{home}, income and birth dates \citep{birth}, is available for every individual. An adversary with access to such auxiliary information \emph{can} learn about an individual's private data from several \emph{other} participants' noisy responses. We illustrate this as follows.%\vspace{-0.1cm}

\begin{tcolorbox}\vspace{-0.25cm}
\textbf{Problem.} An analyst runs a medical survey in Alice's community to investigate how the prevalence of a highly contagious disease changes from neighborhood to neighborhood. Community members report a binary value indicating whether they have the disease.  \vspace{-0.2cm}
\end{tcolorbox}\vspace{-0.1cm}
Next, consider the following two data reporting strategies. \vspace{-0.1cm}
\begin{tcolorbox}\vspace{-0.2cm} \textbf{Strategy $\mathbf{1}$.} Each data owner passes their data through an appropriate randomizer (that flips the input bit with some probability) in their local devices and reports the noisy output to the untrusted data analyst. 
\vspace{-0.2cm}
\end{tcolorbox}\vspace{-0.1cm}
\begin{tcolorbox}\vspace{-0.2cm} \textbf{Strategy $\mathbf{2}$.}  The noisy responses from the local devices of each of the data owners %\footnote{The individual  responses are appropriately anonymized by removing all identifiable information such as IP address of packets.} 
are collected by an intermediary trusted shuffler which dissociates the device IDs (metadata) from the responses and  uniformly randomly shuffles them before sending them to the analyst.\vspace{-0.2cm}
\end{tcolorbox}\vspace{-0.1cm}
\textbf{Strategy $\mathbf{1}$} corresponds to the standard \ldp deployment model (for example, Apple and Microsoft's deployments). Here \textit{the order of the noisy responses is informative of the identity of the data owners} -- the noisy response at index $1$ corresponds to the first data owner and so on. Thus, the noisy responses can be directly linked with its associated device/account ID and subsequently, auxiliary information. This puts Alice's data under the threat of inference attacks.
For instance, an adversary\footnote{The analyst and the adversary could be same, we refer to them separately for the ease of understanding.} may know the home addresses of the participants and use this to identify the responses of all the individuals from Alice's household.  
Being highly infectious, all or most of them 
 %\vspace{-0.4cm} 
\begin{figure*}[t]
\begin{minipage}{0.2\linewidth}
\begin{subfigure}[b]{\linewidth}
\centering
    \includegraphics[width=0.9\columnwidth]{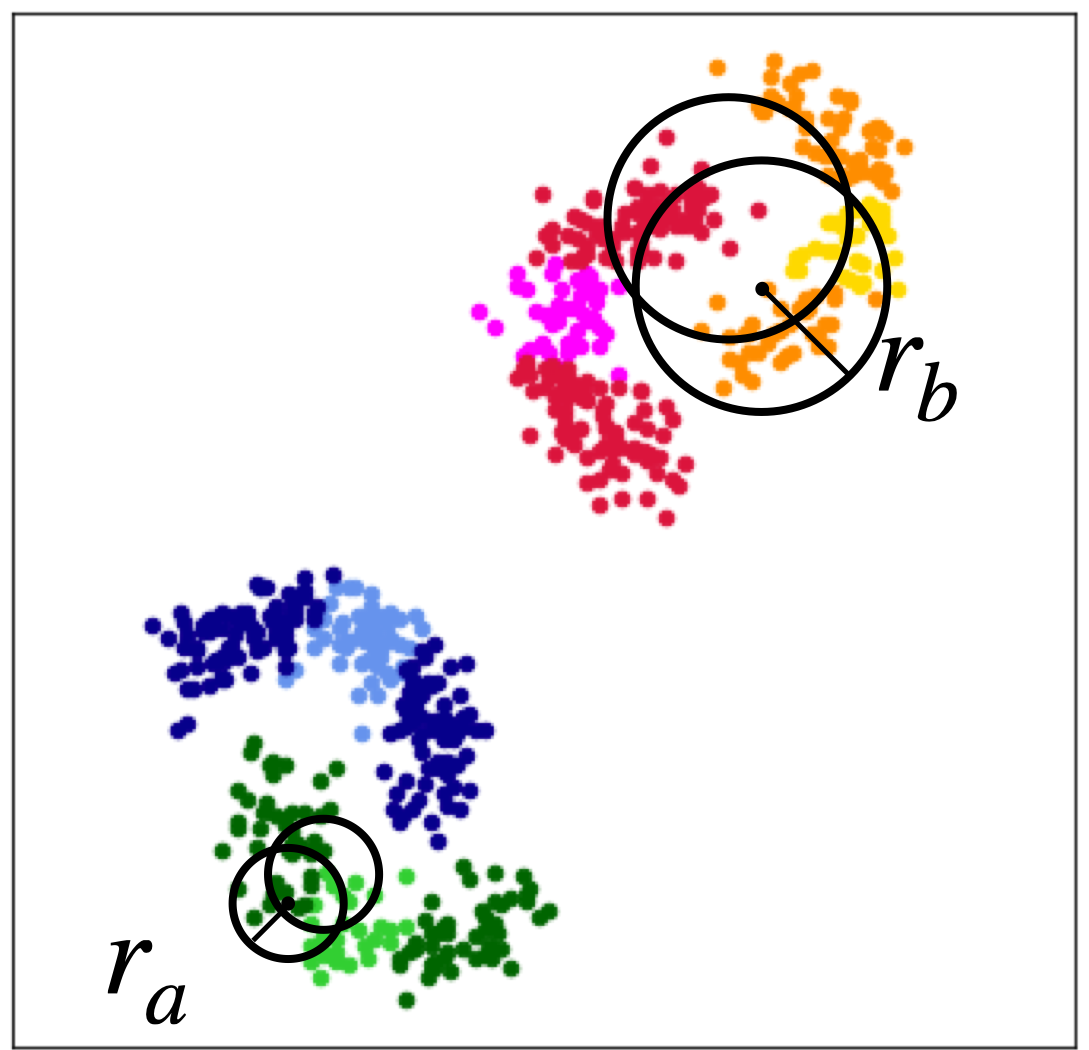}
        \caption{Original Data}
        \label{fig:data}
    \end{subfigure}
    \end{minipage}
    \begin{minipage}{0.8\linewidth}
    \begin{subfigure}[b]{0.25\linewidth}\centering
    \includegraphics[width=0.67\columnwidth]{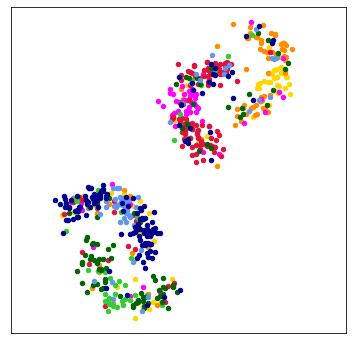}\vspace{-0.25cm}
        \caption{\ldp}
        \label{fig:ldp}
    \end{subfigure}%%
    \begin{subfigure}[b]{0.25\linewidth}\centering
    \includegraphics[width=0.67\columnwidth]{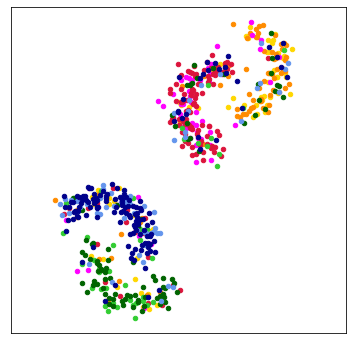}\vspace{-0.25cm}
        \caption{Our scheme: $r_a$}
        \label{fig:r}\end{subfigure}%%
    \begin{subfigure}[b]{0.25\linewidth}\centering
    \includegraphics[width=0.67\columnwidth]{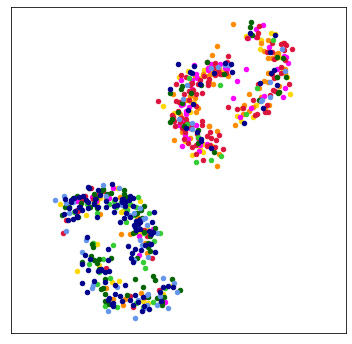}\vspace{-0.25cm}
        \caption{Our scheme: $r_b$}
        \label{fig:R}\end{subfigure}%%
      \begin{subfigure}[b]{0.25\linewidth}\centering
    \includegraphics[width=0.67\columnwidth]{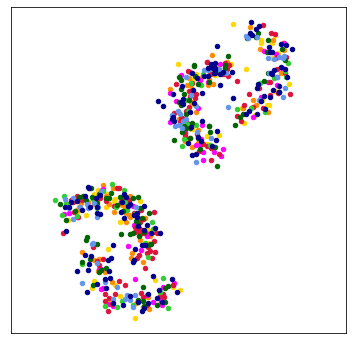}\vspace{-0.25cm}
        \caption{Uniform shuffle}
        \label{fig:uniform}
    \end{subfigure}\\ \vspace{-0.05cm}
    \begin{subfigure}[b]{0.25\linewidth}\centering
    \includegraphics[width=0.67\columnwidth]{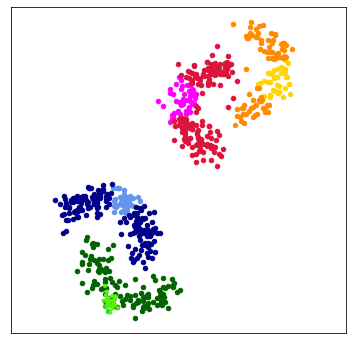}\vspace{-0.25cm}
        \caption{Attack: LDP}
        \label{fig:ldp:attack}
    \end{subfigure}%%
    \begin{subfigure}[b]{0.25\linewidth}\centering
  \includegraphics[width=0.67\columnwidth]{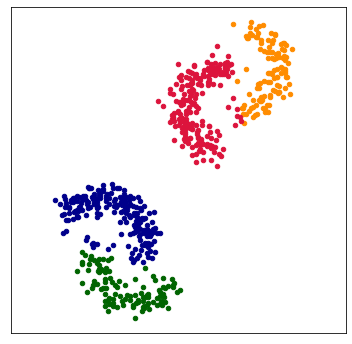}\vspace{-0.25cm}
       \caption{Attack: $r_a$}
        \label{fig:r:attack}\end{subfigure}%%
    \begin{subfigure}[b]{0.25\linewidth}\centering
   \includegraphics[width=0.67\columnwidth]{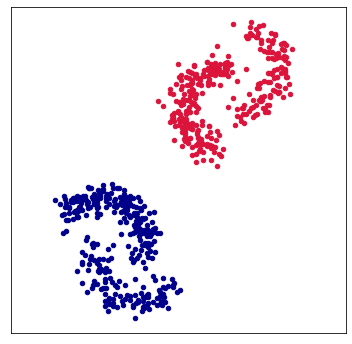}\vspace{-0.25cm}
        \caption{Attack: $r_b$}
        \label{fig:R:attack}\end{subfigure}%%
      \begin{subfigure}[b]{0.25\linewidth}\centering
   \includegraphics[width=0.67\columnwidth]{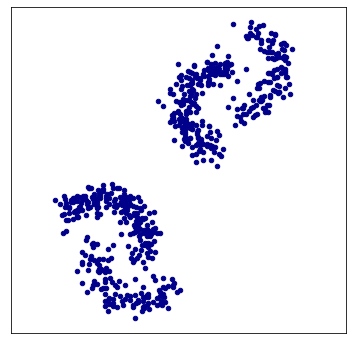}\vspace{-0.25cm}
        \caption{Attack: unif. shuff.}
        \label{fig:uniform:attack}
    \end{subfigure}
 \end{minipage}\vspace{-0.15cm}
    \caption{Demonstration of how our proposed scheme thwarts inference attacks at different granularities. Fig. \ref{fig:data}  depicts the original sensitive data (such as income bracket) with eight color-coded labels. The position of the points represents public information (such as home address) used to correlate them. There are three levels of granularity: warm vs. cool clusters, blue vs. green and red vs. orange crescents, and light vs. dark within each crescent. Fig. \ref{fig:ldp} depicts \scalebox{0.9}{$\epsilon = 2.55$} \textsf{LDP}. Fig. \ref{fig:r} and \ref{fig:R} correspond to our scheme, each with $\alpha = 1$ (privacy parameter, Def. \ref{def:privacy}). The former uses a smaller  distance threshold ($r_1$, used to delineate the granularity of grouping -- see Sec. \ref{sec:privacy:def}) that mostly shuffles in each crescent. The latter uses a larger distance threshold ($r_2$) that shuffles within each cluster. Figures in the bottom row  demonstrate an inference attack (uses Gaussian process correlation) %with a fixed length scale) 
    on all four cases. We see that  \ldp  reveals almost the entire dataset (Fig. \ref{fig:ldp:attack}) while uniform shuffling prevents all classification (\ref{fig:uniform:attack}). However, the granularity can be controlled with our scheme (Figs. \ref{fig:r:attack}, \ref{fig:R:attack}). \vspace{-2em}}
   \label{fig:demonstration}
%   \vspace{-0.15cm}
    \end{figure*}
  will have the same true value ($0$ or $1$). Hence, the adversary can reliably infer Alice's value by taking a simple majority vote of her and her household's noisy responses. Note that this does not violate the \ldp guarantee since the inputs are appropriately randomized when observed in isolation. Additionally, on account of being public, the auxiliary information is known to the adversary (and analyst) \emph{a priori} -- no mechanism can prevent their disclosure. For instance, any attempts to include Alice's address as an additional feature of the data and then report via \ldp is \emph{futile} --   the adversary would simply discard the reported noisy address and use the auxiliary information about the exact addresses to identify the responses of her household members. We call such threats \emph{inference attacks} -- recovering an individual's private input using all or a subset of other participants' noisy responses. It is well known that protecting against inference attacks that rely on underlying data correlations is beyond the purview of \DP  \citep{Pufferfish, sok}. 
  
\textbf{Strategy 2} corresponds to the recently introduced shuffle \DP model, such as Google's Prochlo \citep{Prochlo}.  Here, the noisy responses are completely anonymized -- the adversary cannot identify which \ldp responses correspond to Alice and her household.  Under such a model, only information that is completely order agnostic (i.e., symmetric functions that can be computed over just the \textit{bag} of values, such as aggregate statistics) can be extracted. Consequently, the analyst also fails to accomplish their original goal as all the underlying data correlation is destroyed.
 
Thus, we see that the two models of deployment for \ldp present a trade-off between vulnerability to inference attacks and scope of data learnability. In fact, as demonstrated in \cite{Kifer}, it is impossible to defend against \emph{all} inference attacks while simultaneously maintaining utility for learning. In the extreme case that the adversary knows \emph{everyone} in Alice's community has the same true value (but not which one), no mechanism can prevent revelation of Alice's datapoint short of destroying all utility of the dataset. This then begs the question: \textbf{\emph{Can we formally suppress \underline{specific} inference attacks targeting each data owner while maintaining some meaningful learnability of the private data?}} Referring back to our example, can we thwart attacks inferring Alice's data using specifically her households' responses and still allow the medical analyst to learn its target trends? Can we offer this to every data owner participating?
 
%In this paper, we strike a balance and we propose a generalized shuffle framework for deployment that can interpolate between the two extremes.
In this paper, we strike a balance and propose a generalized shuffle framework that meets the utility requirements of the above analyst while formally protecting data owners against inference attacks. 
Our solution is based on the key insight: \textit{the order of the data acts as the proxy for the identity of data owners} as illustrated above. The granularity at which the ordering is maintained formalizes resistance to inference attacks while retaining some meaningful learnability of the private data. Specifically, we guarantee each data owner that their data is shuffled together with a carefully chosen group of other data owners. Revisiting our example, consider uniformly shuffling the responses from Alice's household and her immediate neighbors. Now an adversary cannot use her household's responses to predict her value any better than they could with a random sample of responses from this group. 
In the same way that \ldp prevents reconstruction of her datapoint using specifically \emph{her} noisy response, this scheme prevents reconstruction of her datapoint using specifically \emph{her households'} responses. The real challenge is offering such guarantees \textit{equally} to \textit{every} data owner. Bob, Alice's neighbor, needs his households' responses shuffled in with his neighbors, as does Luis who is a neighbor of Bob but \textit{not} of Alice. Thus, we have $n$ data owners with $n$ distinct, overlapping groups. Our scheme supports arbitrary groupings (overlapping or not), introducing a diverse and tunable class of privacy/utility trade-offs which is not attainable with either \ldp or uniform shuffling alone.
%This disallows the trivial strategy of shuffling the noisy responses of each group uniformly. To this end, we propose shuffling the responses in a systematic manner that tunes the privacy guarantee, trading it off with data learnability. 
For the above example, our scheme can formally protect each data owner from inference attacks using specifically their household, while still learning how disease prevalence changes across the neighborhoods of Alice's community.\\
This work offers two key contributions to the machine learning privacy literature: 
 \vspace{-0.2cm}
\squishlist
    \item \textbf{Novel privacy guarantee.} We propose a novel privacy definition, \name-privacy that captures the privacy of the \textit{order} of a data sequence (Sec. \ref{sec:privacy:def}) and formalizes the degree of resistance against inference attacks (Sec. \ref{sec:privacy:implications}). \name-privacy allows assigning an arbitrary group, $G_i$, to each data owner, $\DO_i, i \in [n]$. For instance, the groups can represent individuals in the same age bracket, `friends' on social media, or individuals living in each other's vicinity (as in case of Alice in our example).  Recall that the order is informative of the data owner's identity. Intuitively,  \name-privacy protects $\DO_i$ from inference attacks that arise from knowing the \textit{identity} of the members of their group $G_i$ (Sec. \ref{sec:privacy:implications}). %\textcolor{blue}{By shuffling within each group, \name-privacy protects $\DO_i$ against inference attacks that utilize the data of any subset of the members of $G_i$, while still revealing the statistics of each group to the analyst.} 
    %The group assignment is based on a public auxiliary information  -- individuals of a single group are `similar' w.r.t the auxiliary information. 
    Additionally, this grouping determines a threshold of learnability -- any learning that is order agnostic within a group (disease prevalence in a neighborhood -- the data analyst's goal in our example) is utilitarian and allowed; whereas analysis that involves identifying the values of individuals within a group (disease prevalence within specific households -- the adversary's goal) is regarded as a privacy threat and protected against.
 See Fig. \ref{fig:demonstration} for a toy demonstration of how our guarantee allows \textit{tuning the granularity at which trends can be learned}. 
    
    \item \textbf{Novel shuffle framework.} We  propose a novel mechanism that shuffles the data systematically and achieves \name-privacy. This 
   provides a generalized shuffle framework that interpolates between no shuffling (\textsf{LDP}) and uniform random shuffling (shuffle model) in terms of protection against inference attacks and data learnability.
    %\textcolor{blue}{allows analysis of subsets of the reported data (such as within neighborhoods across a city) while 
    %Our experimental results (Sec. \ref{sec:eval}) demonstrates its efficacy against realistic inference attacks.  
\squishend 
%   \vspace{-0.4cm} 
%   \arc{TO-DOs \\1)Include discussion about viewing a sequence as <bag of val, order> 2) discussion about side information - why is it immutable 3) highlight privacy guarantee - users choose groups which should be overlapping for generality  }
\section{Related Work}\label{sec:related_work} 
%\vspace{-0.3cm} 
The shuffle model of \DP \citep{Bittau2017,shuffle2,shuffling1} differs from our scheme as follows. These works $(1)$ study \DP benefits of shuffling whereas we study the inferential privacy benefits, and $(2)$ only study uniformly random shuffling where ours generalizes this to tunable, non-uniform shuffling (see App. \ref{app:related}). 
% Where these works study the differential private benefits of shuffling (lower $\epsilon$), our work studies its inferential privacy benefits. Furthermore, those works only consider uniform random shuffling of the dataset, where our approach  allows configurable, non-uniform shuffling distributions that allow us to tune the tradeoff between inferential risk and the granularity of learnable trends.
% These works differ from our approach in two ways. First, they only study shuffling as a means to amplify local differentially private guarantees, effectively offering a lower $\epsilon$ when viewed from the alternative, central \DP model. Our results cater to local \emph{inferential} privacy (Sec. \ref{sec:ip}), limiting what can be inferred about one individual from other individuals' responses. Second, the shuffle model only studies uniform random shuffling of the dataset. In contrast, our approach allows configurable, non-uniform shuffling distributions that allow us to tune the tradeoff between inferential risk and the granularity of learnable trends. 
\\A steady line of work has studied inferential privacy  \citep{semantics, Kifer,  IP, Dalenius:1977, dwork2010on, sok}. %Kifer et al. \cite{Kifer} formally studied privacy degradation in the face of data correlations and later proposed a  privacy framework, Pufferfish \cite{Pufferfish, Song,Blowfish}, for analyzing inferential privacy. 
%Subsequently, several other privacy definitions have also been proposed for the inferential privacy setting \cite{DDP,BDP,correlated,correlated2,CWP}. 
Our work departs from those in that we focus on \emph{local} inferential privacy and do so via the new angle of shuffling. 
\\Older works such as $k$-anonymity \citep{kanon},  $l$-diversity \cite{ldiv}, Anatomy \citep{anatomy} and others \citep{older1, older2, older3, older4, older5} have studied the privacy risk of non-sensitive auxiliary information or `quasi identifiers'. These works $(1)$ focus on the setting of dataset release, whereas we focus on dataset collection, and $(2)$ do not offer each data owner formal inferential guarantees, whereas we do. %As such, QIs can be manipulated and controlled, whereas we place no restriction on the amount or type of auxiliary information accessible to the adversary, nor do we control it. 
%Additionally, our work offers each individual formal inferential guarantees against informed adversaries, whereas those works do not. 
The De Finetti attack \citep{definetti} shows how shuffling schemes are vulnerable to inference attacks that correlate records to recover the original permutation of sensitive attributes. A strict instance of our privacy guarantee can thwart such attacks (at the cost of no utility, App. \ref{app:de finetti}). 
% \vspace{-0.4cm}
\section{Background}\label{sec:background} 
% \vspace{-0.2cm}
\textbf{Notations.} \textbf{Boldface} (such as $\bx=\langle x_1, 
 \cdots, x_n\rangle$) denotes a data sequence (ordered list); normal font (such as $x_1$) denotes individual values and $\{\cdot\}$ represents a multiset or bag of values.\vspace{-0.3cm}
\subsection{Local Differential Privacy}\label{sec:ldp}
 \vspace{-0.2cm}The local model consists of a set of data owners and an untrusted data aggregator (analyst); each individual perturbs their data using a \ldp algorithm (randomizers) and sends it to the analyst. % which uses these noisy responses to glean information about the entire dataset.
The \ldp guarantee is formally defined as
\begin{defn}\vspace{-0.1cm}[Local Differential Privacy, \ldp \cite{Warner,Evfimievski:2003:LPB:773153.773174,Kasivi}]
 A randomized algorithm $\mathcal{M} : \mathcal{X} \rightarrow \mathcal{Y}$ is $\epsilon$-locally differentially private (or $\epsilon$-\ldp), if for any pair of private values $x, x' \in \mathcal{X}$ and any subset of output, %$\mathcal{W} \subseteq \mathcal{Y}$ we have $\mathrm{Pr}\big[\mathcal{M}(x) \in \mathcal{W}\big] \leq e^{\epsilon} \cdot \mathrm{Pr}\big[\mathcal{M}(x') \in \mathcal{W}  \big]$.
 %\end{defn}
  \vspace{0.1cm}
 \begin{gather}
 \mathrm{Pr}\big[\mathcal{M}(x) \in \mathcal{W}\big] \leq e^{\epsilon} \cdot \mathrm{Pr}\big[\mathcal{M}(x') \in \mathcal{W}  \big]
  \end{gather}
  \end{defn}
%   \vspace{-0.2cm}
The shuffle model is an extension of the local model where the data owners first randomize their inputs. Additionally, an intermediate trusted shuffler applies a \textit{uniformly random permutation} to all the noisy responses before the analyst can view them. The anonymity provided by the shuffler requires less noise than the local model for achieving the same
privacy.  \vspace{-0.2cm}
\subsection{Mallows Model} 
% \vspace{-0.2cm}
\label{sec:background:MM}
A permutation of a set $S$ is a bijection $S\mapsto S$. The set of permutations of $[n], n \in \mathbb{N}$ forms a symmetric group $\mathrm{S}_n$. As a shorthand, we use $\sigma(\bx)$ to denote applying permutation $\sigma \in \mathrm{S}_n$ to a data sequence $\bx$ of length $n$. Additionally, $\sigma(i), i\in [n], \sigma \in \mathrm{S}_n$ denotes the value at index $i$ in $\sigma$ and $\sigma^{-1}$ denotes its inverse. For example, if $\sigma=( 1 \: 3 \: 5\: 4\: 2)$ and $\bx=\langle 21, 33, 45, 65 , 67\rangle$, then $\sigma(\bx)=\langle 21, 45, 67, 65, 33\rangle$, $\sigma(2)=3, \sigma(3)=5$ and $\sigma^{-1}=(1 \: 5 \: 2 \: 4 \: 3)$.
\\Mallows model is a popular probabilistic model for permutations \citep{MM}.  %It is an exponential location model and is usually referred to as the Gaussian distribution for permutations. 
The mode of the distribution is given
by the reference permutation $\sigma_0$ -- the probability of a permutation increases as we move `closer' to $\sigma_0$ as measured by rank distance metrics, such as the Kendall's tau distance (Def. \ref{def:kendall}). The dispersion parameter $\theta$ controls how fast this increase happens.  %The formal definition of the Mallows model is as follows.
\begin{defn}
\label{def: mallows}
For a dispersion parameter $\theta$, a reference permutation $\sigma_o \in \mathrm{S}_n$, and a rank distance measure $\textswab{d}: \mathrm{S}_n \cross \mathrm{S}_n \mapsto \R $,   
$
\mathbb{P}_{\Theta,\textswab{d}}(\sigma:\sigma_0)=\frac{1}{\psi(\theta,\textswab{d})} e^{-\theta  \textswab{d}(\sigma, \sigma_0)}$
is the Mallows model where $\psi(\theta,\textswab{d})=\sum_{\sigma \in \mathrm{S}_n} e^{-\theta \textswab{d}(\sigma,\sigma_0)}$ is a normalization term and  $\sigma \in \mathrm{S}_n$.
\end{defn}

\begin{wrapfigure}{r}{0.4\linewidth}
    \vspace{-1.5cm}
    % \hspace{-1cm}
    \centering
    \includegraphics[height=2.8cm]{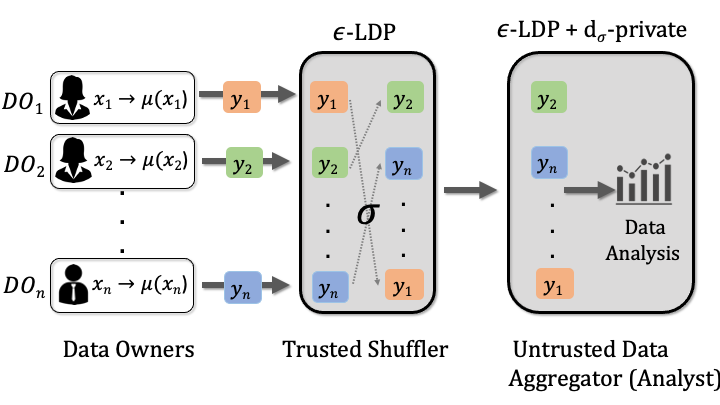} 
    \vspace{-1em}
    \caption{\small{Trusted shuffler mediates on $\by$}} 
    \vspace{-2.5em}
    \label{fig:problemsetting}
\end{wrapfigure}
\vspace{-0.5cm}
\section{Data Privacy and Shuffling}
% \vspace{-0.2cm}
In this section, we present \name-privacy and a shuffling mechanism capable of achieving the \name-privacy guarantee. 
%First, we describe the problem setting. Next, we present our novel privacy definition,  \name-privacy, followed by a semantic understanding of its privacy implications. A utility metric for shuffling mechanisms is presented next. Finally, we introduce a novel shuffling  mechanism capable of achieving the \name-privacy guarantee. 
% \vspace{-0.3cm}
\subsection{Problem Setting} \vspace{-0.1cm}

In our problem setting, we have $n$ data owners $\DO_i, i \in [n]$ each with a private input $x_i \in \mathcal{X}$ (Fig. \ref{fig:problemsetting}).  
 The data owners first randomize their inputs via a $\epsilon$-\ldp mechanism to generate $y_i=\mathcal{M}(x_i)$. 
 %I think we can cut out anything about the adv having side information until the experiments? It's just distracting reviewers? 
 %We consider an informed adversary with public auxiliary information $\mathbf{t}=\langle t_1, \cdots, t_n \rangle, t_i \in \mathcal{T}$ about each individual. 
 Additionally, just like in the shuffle model, we have a trusted shuffler. It mediates upon the noisy responses $\mathbf{y}=\langle y_1,\cdots,y_n \rangle$ 
 %and systematically shuffles them based on $\bt$ (since $\bt$ is public, it is also accessible to the shuffler) 
 to obtain the final output sequence $\bf{z}=\mathcal{A}(\bf{y})$ ($\mathcal{A}$ corresponds to Alg. 1) 
 which is sent to the untrusted data analyst. The shuffler can be implemented via trusted execution environments (TEE) just like Google's Prochlo. 
 %The underlying data correlations in $\mathbf{t}$ is modeled as a prior distribution $\calP$ on $\bf{x}$. 
 Next, we formally discuss the notion of order and its implications. 
\begin{defn}(Order) The order of a sequence $\mathbf{x}=\langle x_1,\cdots, x_n\rangle$ refers to the indices of its set of values $\{x_i\}$ and is represented by permutations from $\mathrm{S}_n$.\vspace{-0.2cm}\end{defn} 
% \begin{figure}
%     \centering
%     \begin{subfigure}[b]{0.49\textwidth}
%          \centering
%          \includegraphics[height = 3.5cm]{shuffle_image.png}
%          \caption{Trust model (similar to shuffle model)}
%          \label{fig:problemsetting}
%      \end{subfigure}
%      \begin{subfigure}[b]{0.49\textwidth}
%          \centering
%          \includegraphics[height = 3.5cm]{graph.png}
%          \caption{An example social media connectivity graph $\bt_{e.g}$}% that acts as the public auxiliary information.}
%          \label{fig:example}
%      \end{subfigure}
%      \caption{how data is privately collected (a) in the face of auxiliary information (b) that can be leveraged to correlate \ldp responses $\langle y_1, \dots, y_n \rangle$}
% \end{figure}

When the noisy response sequence $\mathbf{y}=\langle y_1, \cdots, y_n\rangle$ is represented by the identity permutation $\sigma_{I}=(1 \: 2 \: \cdots \: n)$, the value at index $1$ corresponds to $\DO_1$ and so on. Standard \ldp releases the identity permutation w.p. 1. The output of the shuffler, $\bf{z}$, is some permutation of the sequence $\bf{y}$, i.e.,
% \vspace{-0.5em}
\begin{align*}
\mathbf{z}=\sigma(\by)=
\langle y_{\sigma(1)},\cdots,y_{\sigma(n)}\rangle
%\vspace{-1em}
\end{align*}
where $\sigma$ is determined via $\calA(\cdot)$. For example, for $\sigma=(4 \: 5\: 2 \:3 \: 1)$, we have $\mathbf{z}=\langle y_4, y_5, y_2, y_3, y_1\rangle$ which means that the value at index $1$ ($\DO_1$) now corresponds to that of $\DO_4$ and so on.
%\arc{define analyst ' goal concretely here} \vspace{-0.3cm}

%\textcolor{blue}{The shuffler's distribution over permutations $\sigma$ satisfies \name-privacy by thoroughly shuffling within each data owner's group $G_i$ (formalized in the following section). The analyst, meanwhile, wishes to indefinitely query the statistics of \ldp values at the level of a group or union of groups. To enhance their utility, we optimize the mechanism to preserve statistics within groups (i.e. minimize the extent to which \ldp values shuffle between groups). More on this in Section \ref{sec:mechanism}. In this way, every choice of grouping results in a unique privacy/utility tradeoff. }
%   \vspace{-0.1cm}
\subsection{Definition of \name-privacy}\label{sec:privacy:def}
%  \vspace{-0.3cm}
% \arc{Make groups more general}
Inferential risk captures the threat of an adversary who infers $\DO_i$'s private $x_i$ using all or a subset of other data owners' released $y_j$'s. Since we cannot prevent all such attacks and maintain utility, our aim is to formally limit \emph{which data owners} can be leveraged in inferring $\DO_i$'s private $x_i$. To make this precise, each $\DO_i$ may choose a corresponding group, $G_i \subseteq [n]$, of data owners.
\begin{wrapfigure}{r}{0.3\linewidth}
    \centering
    \vspace{-1.5em}
    \includegraphics[height=2cm]{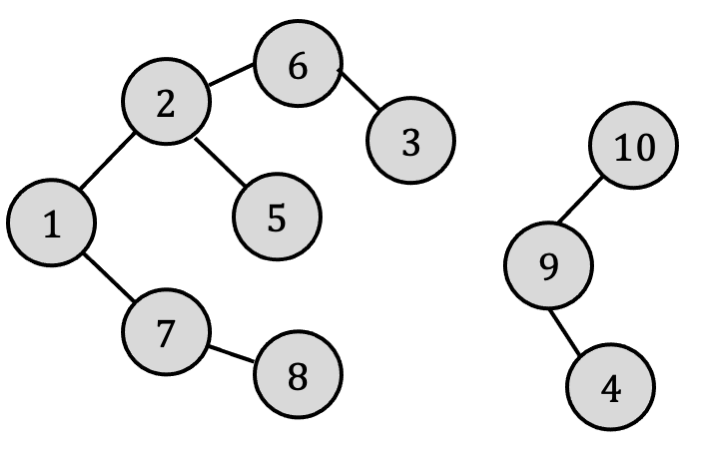}
    \vspace{-1.25em}
    \caption{An example social media connectivity graph $\bt_{e.g}$}
    \vspace{-1em}
    \label{fig:example}
\end{wrapfigure}
 \name-privacy guarantees that $y_j$ values originating from a data owner's group $G_i$ are shuffled together. In doing so, the \ldp values corresponding to subsets of $\DO_i$'s group $I \subset G_i$ cannot be reliably identified, and thus cannot be singled out to make inferences about $\DO_i$'s $x_i$. If Alice's group includes her whole neighborhood, \ldp data originating from her household cannot be singled out to recover her private $x_i$. %We formalize this guarantee in Sec. \ref{sec:privacy:implications}. 
\\Any choice of grouping $\calG = \{G_1, G_2, \dots, G_n\}$ can be accommodated under \name-privacy. Each data owner may choose a group large enough to hide anyone they feel sufficient risk from.  We outline two systematic approaches to assigning groups as follows: %If one feels inferential risk from their close friends, perhaps their group ought to include all of their first-order social media connections. If one feels inferential risk from their close colleagues, perhaps their group should include their entire company. In turn, the analyst may still access aggregate statistics roughly at the group level. 
   \vspace{-0.1cm}\squishlist    \vspace{-0.1cm}
\item Let $\mathbf{t}=\langle t_1, \cdots, t_n \rangle, t_i \in \mathcal{T}$ denote some public auxiliary information about each individual. $\DO_i$'s group, $G_i$, could consist of all those $\DO_j$'s who are similar to $\DO_i$ w.r.t. the public auxiliary information $t_i, t_j$ according to some distance measure $d:\calT \times \calT \rightarrow \R$. Here, we define `similar' as being under a threshold\footnote{We could also have different thresholds, $r_i$, for every data owner, $\DO_i$.} $r \in \R$ such that $G_i = \{j \in [n] \big| d(t_i,t_j) \leq r\},     \forall i \in [n]$.
For example, $d(\cdot)$ can be Euclidean distance if $\calT$ corresponds to geographical locations, thwarting inference attacks leveraging one's household or immediate neighbors.
If $\calT$ represents a social media connectivity graph, $d(\cdot)$ can measure the path length between two nodes, thwarting inference attacks using specifically one's close friends. For the example social media connectivity graph depicted in Fig. \ref{fig:example}, assuming distance metric path length and $r=2$, the groups are defined as  $G_1=\{1,7,8,2,5,6\}, G_2=\{2,1,7,5,6,3\}$ and so on. 

\item  Alternatively, the data owners might opt for a group of a specific size $r < n$. Collecting private data from a social media network, we may set $r = 50$, where each $G_i$ is encouraged to include the $50$ data owners $\DO_i$ interacts with most frequently. 
\squishend 
   \vspace{-0.3cm}
Intuitively, \name-privacy protects $\DO_i$ against inference attacks that leverages correlations at a finer granularity than $G_i$. In other words, under \name-privacy, one subset of $k$ data owners $\subset G_i$ (e.g. household) is no more useful for targeting $x_i$ than any other subset of $k$ data owners $\subset G_i$ (e.g. some combination of neighbors). 
This leads to the following key insight for the formal privacy definition. 

\textbf{Key Insight.} Formally, our privacy goal is to prevent the leakage of ordinal information from within a group. We achieve this by  systematically \textit{bounding the dependence of the mechanism's output on the relative ordering (of data values corresponding to the data owners) within each group}. \\First, we introduce the notion of neighboring permutations. 

%By (non-uniformly) shuffling within each group $G_i \in \calG$, we prevent an adversary from learning whether a set of $k$ \ldp values from $G_i$ correspond to one subset within $G_i$ or another. An adversary cannot distinguish whether these $k$ values came from $\DO_i$'s close friends vs. their distant relatives or from their close neighbors vs. residents on the other side of town. However, an analyst can still observe how the distribution of \ldp values changes across neighborhoods or social circles. 
\begin{defn}   (Neighboring Permutations) Given a group assignment $\mathcal{G}$,  two permutations \scalebox{0.9}{$\sigma, \sigma' \in \mathrm{S}_n$}  are defined to be neighboring w.r.t. a group $G_i \in \calG$ (denoted as \scalebox{0.9}{$\sigma\hspace{-0.1cm} \approx_{G_i}\hspace{-0.1cm} \sigma'$}) if  $\sigma(j) = \sigma'(j) \ \forall j \notin G_i$.
\end{defn} \vspace{-0.25cm}
% \begin{gather} 
% \vspace{-0.2cm}
% \sigma(j) = \sigma'(j) \ \forall j \notin G_i 
% \vspace{-0.4cm}
% \end{gather} 
% \end{defn}
 %\vspace{-0.2cm}
Neighboring permutations differ only in the indices of its corresponding group $G_i$.
 For example, \scalebox{0.9}{$\sigma=(\underline{1} \:\underline{ 2} \: 4 \: 5 \: \underline{7} \: \underline{6} \: \underline{10} \: \underline{3} \: 8 \:9)$} and \scalebox{0.9}{$\sigma'=(\underline{7} \:\underline{3} \: 4 \: 5 \: \underline{6} \: \underline{2} \: \underline{1} \: \underline{10} \: 8 \: 9 )$} are neighboring w.r.t \scalebox{0.9}{$G_1$} (Fig. \ref{fig:example}) since they differ only in \scalebox{0.9}{$\sigma(1), \sigma(2), \sigma(5), \sigma(6), \sigma(7)$} and \scalebox{0.9}{$\sigma(8)$}. We denote the set of all neighboring permutations as %\vspace{-0.2cm}
    \vspace{0.1cm}\begin{gather}      \mathrm{N}_{\calG}=\{(\sigma,\sigma')|\sigma \approx_{G_i} \sigma', \exists G_i \in \calG \}     \vspace{0.2cm}\end{gather}    \vspace{0.2cm}
Now, we formally define \name-privacy as follows.
 \vspace{-0.1cm}\begin{defn}[\name-privacy] For a given group assignment $\calG$ on a set of $n$ entities and a privacy parameter $\alpha \in \R_{\geq0}$, a randomized  mechanism $\calA: \mathcal{Y}^n \mapsto \mathcal{V} $ is $(\alpha,\mathcal{G})$-\name~private if for all $\mathbf{y} \in \mathcal{Y}^n$ and neighboring permutations $\sigma, \sigma' \in \mathrm{N}_\calG$ and any subset of output $O\subseteq \mathcal{V}$, we have
\vspace{0.1cm} 
\begin{gather*} 
    %\vspace{-0.5cm}
    \mathrm{Pr}[\calA\big(\sigma(\mathbf{y})\big) \in O] \leq e^\alpha \cdot \mathrm{Pr}\big[\calA\big(\sigma'(\mathbf{y})\big) \in O \big] \numberthis \label{eq:privacy} 
\end{gather*}   
%where $\bz=\pi(\by)=\langle y_{\pi(1)},\cdots, y_{\pi(n)}\rangle, \pi \in \mathrm{S}_n$
 $\sigma(\mathbf{y})$ and $\sigma'(\mathbf{y})$  are defined to be \textit{neighboring sequences}. 
\label{def:privacy}\end{defn} \vspace{-0.2cm}
 \name-privacy states that, for any group $G_i$,  the mechanism is (almost) agnostic of the order of the data within the group.  Even after observing the output, an adversary cannot learn about the relative ordering of the data within any group. Thus, two neighboring sequences are indistinguishable to an adversary. %For example, for any data sequence $\mathbf{y}\in \mathcal{Y}^{10}$, $\sigma=(1 \: 2 \: 4 \: 5 \: 7 \: 6 \: 10 \: 3 \: 8 \:9)$ and $\sigma'=(7 \:3 \: 4 \: 5 \: 6 \: 2 \: 1 \: 10 \: 8 \: 9 )$, $\sigma(\mathbf{y})$ and $\sigma'(\mathbf{y})$ are indistinguishable to an adversary ($\sigma\approx_{G_1}\sigma'$ for Fig. \ref{fig:example}). 
 %In other words, a \name-private mechanism is (almost) order-agnostic for any group in $\mathcal{G}$. 
An important property of \name-privacy is that post-processing computations does not degrade privacy. Additionally, when applied multiple times, the privacy guarantee degrades gracefully. Both the properties are analogous to \DP and are presented in App. \ref{app:post-processing}. %Interestingly, \ldp mechanisms achieve a weak degree of \name-privacy. 
 \begin{comment}
    
\begin{lem} 
    An $\epsilon$-\ldp mechanism is $(k\epsilon, \calG)$-\name~ private for any group assignment $\calG$ such that $
        k \geq \max_{G_i \in \calG} |G_i|
$ (proof in App. \ref{app:post-processing}).\label{lemma:LDP} \vspace{-0.2cm}
\end{lem} 
 \end{comment}

\textbf{Note.} Any data sequence $\mathbf{x}=\langle x_1,\cdots, x
_n\rangle$ can be viewed as a two-tuple,  $\big(\{x\}, \sigma\big)$, where $\{x\}$ denotes the \textit{bag} of values and $\sigma \in S_n$ denotes the corresponding indices of the values which represents the \textit{order} of the data.
 The $\epsilon$-LDP protects the bag of data values, $\{x\}$, while $d_\sigma$-privacy protects the order, $\sigma$. Thus, the two privacy guarantees cater to orthogonal parts of a data sequence (see Thm. \ref{thm: decision theoretic} ). Also, $\alpha=\infty~(0), r = 0~(n)$ represents the standard $\ldp$(shuffle \textsf{DP}) setting.

%The proof of the above lemma is presented in App. \ref{app}. 
%So, if a mechanism satisfies $(\epsilon = 2)$-\ldp, then it also satisfies $(10,\calG)$ \name-privacy for any group assignment $\calG$ whose largest group contains $5$ individuals.

%  \vspace{-1em}
\subsection{Privacy Implications}
\label{sec:privacy:implications}
% \vspace{-0.2cm} 
% In this section, we describe the implications of the \name-privacy guarantee in our setting. As discussed above, \name-privacy  delineates the \textit{granularity at which the underlying data correlation can be leveraged} by the adversary. Specifically, the group assignment $\mathcal{G}$ delineates a threshold of learnability as follows
%We now turn to \name-privacy's semantic guarantees: what can/cannot be learned from the released sequence $\bz$? 
The group assignment $\mathcal{G}$ delineates a threshold of learnability which determines the privacy/utility tradeoff as follows.
% As with \DP, the \name-privacy definition alone does not communicate a meaningful notion of privacy. For this, we introduce two semantic guarantees --- limiting what an adversary may do and learn --- that result from satisfying \name-privacy. Each of these is a precise way of stating the same concept: \name-privacy prevents adversaries from reconstructing any $x_i$ using the any specific subset of sanitized $y_j$ values in $\DO_i$'s group, $G_i$. 

%\name-privacy offers a form of local inferential privacy: informed Bayesian adversaries learn very little about $\DO_i$'s private value $x_i$ from the shuffled \sequence $\bz$. 

% \vspace{-0.3cm}
\squishlist
    \item \textbf{Learning allowed (Analyst's goal)}. 
    % Anything that can be learned about $\DO_i$ from (1) the correlations at the granularity of $G_i$, and (2) individuals outside $\DO_i$ is allowed -- this encodes the utility of the data from the analyst's perspective. The former allows learning of information which can be extracted from just the \textit{bag} of the corresponding (noisy) data values, denoted as  $\{y_{G_i}\}$.  The latter allows learning of information from $\by_{\overline{G}_i}$, the (ordered and noisy) data sequence for all data owners outside $G_i$. The rationale for allowing this is that individuals outside $G_i$ are not similar to $\DO_i$ (w.r.t auxiliary information $\bt$) and hence, not very informative about $\DO_i$ (not a potential privacy threat).
  \name-privacy can answer queries that are order agnostic within groups, such as aggregate statistics of a group. In Alice's case, the analyst can estimate the disease prevalence in her neighborhood. 
    \item  \textbf{Learning disallowed (Adversary's goal)}. 
    %Utilizing correlation within the group $G_i$ (i.e., for a given set of values $\{y_{G_i}\}$, additional information extractable from the order of $\by_{G_i}$) to learn about $\DO_i$ is disallowed. This encodes the privacy threat from the adversary's perspective.
    Adversaries cannot identify (noisy) values of individuals  within any group. While they may learn the disease prevalence in Alice's neighborhood, they cannot determine the prevalence within her household and use that to recover her value $x_i$.
\squishend  
% \vspace{-0.2cm}

%%%%%%%%%%%%%%%%%%%%%%%%%%%%
%BAYESIAN BACKGROUND
%%%%%%%%%%%%%%%%%%%%%%%%%%%%

To make this precise, we first formalize the privacy implications of the \name~guarantee in the standard Bayesian framework, typically used for studying inferential privacy. Next, we formalize the privacy provided by the combination of \ldp and \name~guarantees by way of a decision theoretic adversary. %against $(1)$ a Bayesian adversary trying to infer $\DO_i$'s true value, $x_i$ and $(2)$ a decision theoretic adversary who wants to identify the $z_i$ values corresponding to a given subset of $k$ data owners, such as the $k$ members of Alice's household. \\ 
\\

\textbf{Bayesian Adversary.} Consider a Bayesian adversary with any prior $\calP$ on the joint distribution of noisy responses, $\Pr_\calP[\by]$, which models their beliefs on the correlation between the participants (such as the correlation between Alice and her households' disease status). Their goal is to infer $\DO_i$'s private input $x_i$. As with early \DP works \citep{dwork_early}, we consider an \emph{informed} adversary. Here, the adversary %With \textsf{DP}, informed adversaries know the private input $x_j$ of every data owner $\DO_j$ except $x_i$. 
 knows %$(1)$ the (unordered) bag of noisy values $\{y_{G_i}\}$ in $i$'s group, and $(2)$ the (ordered) sequence of noisy values $\mathbf{y}_{\overline{G}_i}$ outside $i$'s group. 
 $(1)$ the sequence (assignment) of noisy values outside $G_i$, $\by_{\overline{G}_i}$, and $(2)$ the (unordered) bag of noisy values in $G_i$, $\{y_{G_i}\}$. \name-privacy bounds the prior-posterior odds gap on $x_i$ for such as informed adversary as follows:  
\begin{thm}
\label{thm: semantic guarantee}
For a given group assignment $\calG$ on a set of $n$ data owners, if a shuffling mechanism $\calA:\calY^n\mapsto \calY^n$ is $(\alpha,\calG)$-\name private, then for each data owner $\DO_i, i \in [n]$, %\vspace{-0.1cm}
\begin{align*}
   \max_{\substack{i\in [n]\\ a,b \in \calX}} \bigg|\log \frac{\Pr_\calP [x_i = a | \bz, \{y_{G_i}\},\by_{\overline{G}_i}]}{\Pr_\calP [x_i = b | \bz, \{y_{G_i}\},\by_{\overline{G}_i}]} - \log \frac{\Pr_\calP [x_i = a | \{y_{G_i}\},\by_{\overline{G}_i}]}{\Pr_\calP [x_i = b | \{y_{G_i}\},\by_{\overline{G}_i}]} \bigg| \leq \alpha  %\vspace{-0.5cm}
\end{align*}
for a prior distribution $\calP$, where \scalebox{0.9}{$\bz=\calA(\by)$} and \scalebox{0.9}{$\by_{\overline{G}_i}$} is the noisy sequence for data owners outside \scalebox{0.9}{$G_i$}. %(proof in App. \ref{app:thm:semantic}). 
\end{thm}
\vspace{-1em}
% The above privacy loss variable differs slightly from that of Def. \ref{def:ip}, since the informed adversary already knows $\{y_{G_i}\}$ and $\by_{\overline{G}_i}$. %, much like a \DP informed adversary knowing every other datapoint $\bx_{-i}$. 
 %Equivalently, this bounds the prior-posterior odds gap on $x_i$: 
%\vspace{-0.2em}
See App \ref{app:bayesian proof} for the proof and further discussion on the semantic meaning of the above guarantee. \\

\newcommand{\Aunif}{\calA_{\text{unif}}}
\newcommand{\Ashuff}{\calA_{\text{shuff}}}
\newcommand{\Pzo}{P_{0 \rightarrow 1}}
\newcommand{\Poz}{P_{1 \rightarrow 0}} 

\textbf{Decision Theoretic Adversary.} Here, we analyse the privacy provided by the combination of \ldp and \name~guarantees.
%The guarantee above formalizes this by comparing what a Bayesian adversary learns about $x_i$ from $\bz$ to what they learn from the bag of $\ldp$ values inside $\DO_i$'s group and the sequence of \ldp values outside it. 
%We make this concrete by showing that no adversary can reliably find which released %For instance, in combination with a low $\epsilon$, \name-privacy guarantees that no adversary can reliably pick out which $z_i$ values originated from Alice's household, and thus cannot specifically leverage those to make inferences on her. 
Consider a decision theoretic adversary who aims to identify the noisy responses, $\{z_I\}$, that originated from a specific subset of data owners, $I \subset G_i$ (such as the members of Alice's household). %This re-identification attack is a key step in carrying out the inference attack  %observes the output sequence $\bz$ and chooses which indices in $\bz$ correspond to a subset of data owners $I \subset G_i$. 
We denote the adversary by a (possibly randomized) function mapping from the output $\bz$ sequence to a set of $k$ indices, $\mathcal{D}_{Adv}: \calY^n \rightarrow [n]^k$, where $k = |I|$. These $k$ indices, $H \in [n]^k$, represent the elements of $\bz$ that $\mathcal{D}_{Adv}$ believes originated from the data owners in $I$. $\mathcal{D}_{Adv}$ wins if $> k/2$ of the chosen indices indeed originated from $I$, i.e,  $|\sigma(H) \cap I| > k/2 $, where $z_i = y_{\sigma(i)}$ and $\sigma(H) = \{\sigma(i) : i \in H\}$. $\mathcal{D}_{Adv}$ loses if most of $H$ did not originate from $I$, i.e.,  $|\sigma(H) \cap I| \leq k/2 $.
% \footnote{$k \ll r$; the adversary must recover the \ldp values of the entire subgroup $I \subset G_i$ to make inferences on $x_i$.} 
We choose the above adversary because this re-identification is a key step in carrying out inference attacks -- in failing to reliably re-identify the noisy values originating from $I$, one cannot make inferences on $x_i$ specifically from the subset $I \subset G_i$. 

%We assume that $k < \frac{r}{2}$, where $r$ is the size of Alice's group, $r = |G_i|$. 

% $A_I$ wins if at least $l$ of the $k$ indices it selects in $\bz$, $H$, indeed originated from $l$ of the $k$ data owners in $I$. $A_I$ loses if less than $l$ of the $k$ indices it selects originated from data owners in $I$. We assume that $l < \frac{k}{2}$ and that $k < \frac{r}{2}$, where $r$ is the size of the group, $r = |G_i|$. 

% This notion that the \ldp $y_i$ values of a specific subset of $G_i$ cannot be leveraged is made even more concrete by considering 

% Consider the original notation where the mechanism is not split into bag of values and order. We can also provide a guarantee against any adversary trying to identify the indices of $\bz$ originating from data owners $I \subset G_i$ (non differentially). 

% Formally, consider an adversary who tries to identify the \ldp datapoints originating from subset $I \subset G_i$, $A_I:\calZ^{n} \rightarrow [n]^{k}$, where $|I| = k$ and $|G_i| = r > k$. Upon observing any partially-shuffled sequence $\bz$, $A_I$ returns $H\subset [n]$, a set of $k$ indices in $\bz$ which it believes originated from data owners $I$. We lower bound the error rate of $A_I$: 

\begin{thm}
\label{thm: decision theoretic}
   For $\mathcal{A}(\mathcal{M}(\bx))=\bz$ where $\mathcal{M}(\cdot)$ is $\epsilon$-\ldp and $\mathcal{A}(\cdot)$ is $\alpha$ - \name private, we have  
 \begin{align*}
     \Pr[\mathcal{D}_{Adv} \text{ loses}] \geq \big\lfloor \frac{r-k}{k} \big\rfloor e^{-(2k\epsilon+\alpha)} \cdot \Pr[\mathcal{D}_{Adv} \text{ wins}]
 \end{align*}
 for any input subgroup $I \subset G_i, r = |G_i|$ and  $k < r/2$. 
\end{thm}
% \vspace{-0.3cm}
The adversary's ability to re-identify the $\{z_I\}$ values comes partially from the \textit{bag of values } (quantified by $\epsilon$) and partially from the \textit{order} (quantified by $\alpha$). We highlight two implications of this fact. 
% \vspace{-0.4cm}
\squishlist
    \item When $\epsilon$ is small ($\ll 1$), an adversary's ability to re-identify the noisy values $\{z_I\}$ originating from $I$ may very well be dominated by $\alpha$. For instance, if $\epsilon = 0.2$ and $k = 5$, the adversary's advantage is dominated by $\alpha$ for any $\alpha > 2$. When using $\ldp$ alone (no shuffling), $\alpha = \infty$ and the adversary can exactly recover which values came from Alice's household. As such, even a moderate $\alpha$ value (obtained via \name-privacy) significantly reduces the ability to re-identify the values. 
    \item When the loss is dominated by $\epsilon$ ($2k \epsilon \gg \alpha$), the above expression allows us to disentangle the \textit{source of privacy loss}. In this regime, adversaries get most of their advantage from the bag of values released, not from the order of the release. That is, even if $\alpha = 0$ (uniform random shuffling), participants still suffer a large risk of re-identification simply due to the noisy values being reported. Thus, no shuffling mechanism can prevent re-identification in this regime. %One could imagine that if the group size $r$ is small and $\epsilon$ is large, no shuffling mechanism can prevent re-identification if the adversary has a strong prior on the user's values. 
\squishend 
%\vspace{-0.2cm}
\textbf{Discussion.} In spirit, \DP does not guarantee protection against recovering $\DO_i$'s private $x_i$ value. It guarantees that -- had a user not participated (or equivalently submitted a false value $x_i'$) -- the adversary would have about the same ability to learn their true value, potentially from the responses of other data owners. In other words, the choice to participate is unlikely to be responsible for the disclosure of $x_i$. Similarly, \name-privacy does not prevent disclosure of $x_i$. By requiring indistinguishability of neighboring permutations, it guarantees that -- had the data owners of any group $G_i$ completely swapped identities -- the adversary would have about the same ability to learn $x_i$. So most likely, Alice's household is not uniquely responsible for a disclosure of her $x_i$: had her household swapped identities with any of her neighbors, the adversary would probably draw the same conclusion on $x_i$. 
Or, as detailed in Thm.\ref{thm: decision theoretic}, an adversary cannot reliably resolve which $\{z\}$ values originated from Alice's household, so they cannot draw conclusions based on her household's responses. 
In a nutshell,    
% \vspace{-0.4cm}
\squishlist   
% \vspace{-0.3cm}
    \item Inference attacks can recover a data owner $\DO_i$'s private data $x_i$ from the responses of other data owners. The order of the data acts as the proxy for the data owner's identity which can aid an adversary in corralling the subset of other data owners who correlate with $\DO_i$ (required to make a reliable inference of $x_i$). 
    \item  \DP alleviates concerns that \underline{$\DO_i$'s choice to share data} ($y_i$) will result in disclosure of $x_i$, and \name-privacy alleviates concerns that $\DO_i$'s \underline{group's ($G_i$) choice to share their identity} will result in disclosure of $x_i$.
    % \vspace{-0.2cm}
\squishend
\subsection{\name-private Shuffling Mechanism}
\label{sec:mechanism}
%\vspace{-0.2cm}

\begin{wrapfigure}{R}{0.425\textwidth}
    \vspace{-1em}
    \IncMargin{1em}
    \setlength{\textfloatsep}{-1pt}
    \begin{algorithm}[H]
    \caption{\name-private Shuffling Mech.}
    \setstretch{1.25}
    \KwIn{\ldp sequence \scalebox{0.9}{$\by=\langle y_1, \cdots, y_n \rangle$}\;
    \hspace{3em}Public aux. info. $\bt=\langle t_1, \cdots t_n \rangle$\;
    \hspace{3em}Dist. threshold $r$; Priv. param. $\alpha$\;
    }
    %\hspace{1cm} 
    \KwOut{$\bz$ - Shuffled output sequence\;}
    \vspace{0.1em}
    \nl $\calG=$ \scalebox{0.9}{$ComputeGroupAssignment$} $(\bt,r)$\;
    % \hspace{0.7cm}\textcolor{blue}{$\rhd$} Assign groups for each data owner using Eq.~\eqref{eq:group2}\vspace{0.05cm}
    \nl Construct graph $\mathbb{G}$ with \\
     \hspace{1em} a) vertices \scalebox{0.9}{$V=\{1,2,\cdots,n\}$}\\ 
     \hspace{1em} b) edges \scalebox{0.9}{$E = \{(i,j): j \in G_i, G_i \in \calG\}$} \\
    %  \Statex\hfill \textcolor{blue}{$\rhd$} Translate the group assignment $\calG$ to a graph $\mathbb{G}$ 
    \nl \scalebox{0.95}{$root=\arg \max_{i \in [n]} |G_i|$}\;
    %  \Statex\hfill \textcolor{blue}{$\rhd$} \scalebox{0.95}{$root$} corresponds to the node with the largest group \vspace{0.05cm}
    \nl \scalebox{0.95}{$\sigma_0=\textsf{BFS}(\mathbb{G},root)$}\;
    %  \Statex\hfill \textcolor{blue}{$\rhd$} Compute reference permutation via a BFS traversal on $\mathbb{G}$ 
    \nl $\Delta$= \scalebox{0.9}{$ComputeSensitivity$}$(\sigma_0,\calG)\;$
      
    % \hfill\textcolor{blue}{$\rhd$} Using Prop. \ref{prop:1} 
    \nl \scalebox{0.9}{$\theta=\alpha/\Delta$}\; 
    % \hfill\textcolor{blue}{$\rhd$} Compute the dispersion parameter 
    \nl \scalebox{0.9}{$\hat{\sigma} \sim \mathbb{P}_{\theta,\textswab{d}}(\sigma_0) $ }\;
    %  \hfill\textcolor{blue}{$\rhd$} Sampling from the Mallows model
    \nl $\sigma^*=\sigma_0^{-1}\hat{\sigma}$\; 
    \nl $\bz=\langle y_{\sigma^*(1)}, \cdots y_{\sigma^*(n)}\rangle$\; 
    % \hfill\textcolor{blue}{$\rhd$} Shuffle the output   
    \nl Return $\bz$\;
    \label{algo:main}
    \end{algorithm}
    \vspace{-2em}
\end{wrapfigure}
%\vspace{2cm}

We now describe our novel shuffling mechanism that can achieve \name-privacy. In a nutshell, our mechanism samples a permutation from a suitable Mallows model and shuffles the data sequence accordingly. We can characterize the \name-privacy guarantee of our mechanism in the same way as that of the \DP guarantee of classic mechanisms \citep{Dwork} -- with variance and sensitivity. Intuitively, a larger dispersion parameter $\theta \in \R$ (Def. \ref{def: mallows}) reduces randomness over permutations, increasing utility and increasing (worsening) the privacy parameter $\alpha$. The maximum value of $\theta$ for a given $\alpha$ guarantee depends on the sensitivity of the rank distance measure $\textswab{d}(\cdot)$ over all neighboring permutations $N_\calG$. Formally, we define the sensitivity as \\
\resizebox{0.95\linewidth}{!}{
\begin{minipage}{\linewidth}
\begin{align*}
    \Delta(\sigma_0 : \textswab{d}, \calG) =
    \max_{(\sigma, \sigma') \in N_\calG} |\textswab{d}(\sigma_0 \sigma, \sigma_0) - \textswab{d}(\sigma_0 \sigma', \sigma_0)|~, 
\end{align*}
\end{minipage}
}
%\vspace{-0.2cm}
% \begin{prope}
% For group assignment $\calG$, a  mechanism $\calA(\cdot)$ that shuffles according to a permutation sampled from the Mallows model $\mathbb{P}_{\theta,\textswab{d}}(\cdot)$, satisfies $(\alpha, \calG)$-\name privacy where
% \vspace{-0.1cm}
% \begin{gather}
% \vspace{-0.2cm}
%  \hspace{-0.5cm}\Delta(\sigma_0 : \textswab{d}, \calG) = \max_{(\sigma, \sigma') \in N_\calG} |\textswab{d}(\sigma, \sigma_0) - \textswab{d}(\sigma', \sigma_0)|\\
%     \alpha 
%     = \theta \cdot \Delta(\sigma_0 : \textswab{d}, \calG)
% \vspace{-0.2cm} 
% \end{gather} 
% We refer to $\Delta(\sigma_0 : \textswab{d}, \calG) $ as the sensitivity of the rank-distance measure $\textswab{d}(\cdot)$ (details in App. \ref{app:prop}).
% \label{prop:1}
% \end{prope}
% \vspace{-0.2cm}
  the maximum change in distance $\textswab{d}(\cdot)$ from the reference permutation $\sigma_0$ for any pair of neighboring permutations $(\sigma,\sigma') \in N_\calG$ permuted by $\sigma_0$. The privacy parameter of the mechanism is then proportional to its sensitivity \scalebox{1}{$\alpha = \theta \cdot \Delta(\sigma_0 : \textswab{d}, \calG)$}. 
  
  Given $\mathcal{G}$ and a reference permutation $\sigma_0$, the sensitivity of a rank distance measure $\textswab{d}(\cdot)$ depends on the \emph{width}, $\omega_{\calG}^{\sigma}$, which measures how `spread apart' the members of any group of $\mathcal{G}$ are in $\sigma_0$:\vspace{-0.2cm}
 \begin{align*}
     \omega_{G_i}^{\sigma}&= \max_{(j,k) \in G_i \times G_i} \Big| \sigma^{-1}(j) - \sigma^{-1}(k) \Big|, i \in [n];  \\
    \omega_{\calG}^{\sigma} &= \max_{G_i \in \calG} \omega_{G_i}^{\sigma}
     \vspace{-1em}
 \end{align*}
% \begin{defn}(Width) 
% For a permutation $\sigma$, the width of a group assignment, $\omega_{\calG}^{\sigma}$, is defined as the maximum separation in $\sigma$ between any two members of a group in $\calG$. 
%  \begin{gather*}
%  %\vspace{-0.3cm}
%  \hspace{-0.5cm}\omega_{G_i}^{\sigma}= \max_{(j,k) \in G_i \times G_i} \Big| \sigma^{-1}(j) - \sigma^{-1}(k) \Big|, i \in [n];  \hspace{0.5cm}
%     \omega_{\calG}^{\sigma} = \max_{G_i \in \calG} \omega_{G_i}^{\sigma}%\vspace{-0.3cm}
% \end{gather*}
% %\vspace{-0.3cm}
% \label{def:width}
% \end{defn}
% \vspace{-0.4cm}
% $\omega_{G_i}^{\sigma}$ measures how `spread apart' the members of $G_i$ in permutation $\sigma$ are. 
For example, for \scalebox{0.9}{$\sigma=(1\:3\:7\:8\:6\:4\:5\:2\:9\:10)$} and \scalebox{0.9}{$G_1=\{1,7,8,2,5,6\}$}, \scalebox{0.9}{$\omega_{G_1}^{\sigma}=|\sigma^{-1}(1)-\sigma^{-1}(2)|=7$}. The sensitivity is an increasing function of the width. For instance, for Kendall's \scalebox{0.9}{$\tau$} distance \scalebox{0.9}{$\textswab{d}_\tau(\cdot )$} we have \scalebox{0.9}{$\Delta(\sigma_0 : \textswab{d}_\tau, \calG)
    =\omega_{\calG}^{\sigma_0}(\omega_{\calG}^{\sigma_0} + 1)/2$}. \\If a reference permutation clusters the members of each group closely together (low width), then the groups are more likely to permute within themselves. This has two benefits. First, for the same $\theta$ ($\theta$ is an indicator of utility as it determines the dispersion of the sampled permutation), a lower value of width gives lower $\alpha$ (better privacy).  Second, if a group is likely to shuffle within itself, it will have better \scalebox{0.9}{$(\eta, \delta)$}-preservation -- a novel utility metric, we propose, for a shuffling mechanism. Intuitively, a mechanism is $(\eta,\delta)$-preserving w.r.t a subset of indices \scalebox{0.9}{$S \subset [n]$} if at least \scalebox{0.9}{$\eta\%$}  of its indices are shuffled within itself with probability \scalebox{0.9}{$(1-\delta)$}. The rationale behind this metric is that it captures the utility of the learning allowed by \name-privacy -- if \scalebox{0.9}{$S$} is equal to some group \scalebox{0.9}{$G \in \calG$}, high \scalebox{0.9}{$(\eta, \delta)$}-preservation allows overall statistics of \scalebox{0.9}{$G$} to be captured since $\eta\%$ of the correct data values remain preserved.   We present the formal discussion in App. \ref{app:utility}. 

Unfortunately, minimizing $\omega_\calG^\sigma$ is an NP-hard problem (Thm. \ref{thm:NP} in App. \ref{app:NP}). Instead, we estimate the optimal $\sigma_0$ using the following heuristic\footnote{The heuristics only affect $\sigma_0$ (and utility). Once $\sigma_0$ is fixed, $\Delta$ is computed exactly as discussed above.} approach based on a graph breadth first search. \\ 

\textbf{Algorithm Description.}
Alg. 1 above proceeds as follows. We first compute the group assignment, $\calG$, based on the public auxiliary information and desired threshold $r$ following discussion in Sec. \ref{sec:privacy:def} (Step 1). Then we construct $\sigma_0$ with a breadth first search (BFS) graph traversal. 
\\We translate $\calG$ into an undirected graph \scalebox{0.9}{$(V,E)$}, where the vertices are indices \scalebox{0.9}{$V = [n]$} and two indices \scalebox{0.9}{$i,j$} are connected by an edge if they are both in some group (Step 2). Next, \scalebox{0.9}{$\sigma_0$} is computed via a breadth first search traversal (Step 4) --  if the \scalebox{0.9}{$k$}-th node in the traversal is \scalebox{0.9}{$i$}, then \scalebox{0.9}{$\sigma_0(k) = i$}. The rationale is that neighbors of \scalebox{0.9}{$i$} (members of \scalebox{0.9}{$G_i$}) would be traversed in close succession. Hence, a neighboring node \scalebox{0.9}{$j$} is likely to be traversed at some step \scalebox{0.9}{$h$} near \scalebox{0.9}{$k$} which means \scalebox{0.9}{$|\sigma_0^{-1}(i) - \sigma_0^{-1}(j)| = |h - k|$} would be small (resulting in low width). Additionally, starting from the node with the highest degree (Steps 3-4) which corresponds to the largest group in $\calG$ (lower bound for $\omega_{\calG}^{\sigma}$ for any  $\sigma$) helps to curtail the maximum width in $\sigma_0$.
% \vspace{-0.25cm}
% \squishlist 
% \item We translate $\calG$ into an undirected graph $(V,E)$, where the vertices are indices $V = [n]$ and two indices $i,j$ are connected by an edge if they are both in some group (Step 2). Next, $\sigma_0$ is computed via a breadth first search traversal (Step 4) --  if the $k$-th node in the traversal is $i$, then $\sigma_0(k) = i$. The rationale is that neighbors of $i$ (members of $G_i$) would be traversed in close succession. Hence, a neighboring node $j$ is likely to be traversed at some step $h$ near $k$ which means $|\sigma_0^{-1}(i) - \sigma_0^{-1}(j)| = |h - k|$ would be small (resulting in low width).  \vspace{-0.03cm} \item  We start from the node with the highest degree (Steps 3-4) which corresponds to the largest group in $\calG$ (lower bound for $\omega_{\calG}^{\sigma}$ for any  $\sigma$). This is a good heuristic since it curtails the spread out of largest group.  \vspace{-0.25cm}
% \squishend

This is followed by the computation of the dispersion parameter, \scalebox{0.9}{$\theta$}, for our Mallows model (Steps 5-6). 
Next, we sample a permutation from the Mallows model (Step 7) \scalebox{0.9}{$\hat{\sigma} \sim  \mathbb{P}_{\theta}(\sigma:\sigma_0) $} and we apply the inverse reference permutation to it, \scalebox{0.9}{$\sigma^* = \sigma_0^{-1} \hat{\sigma}$} to obtain the desired permutation for shuffling. Recall that $\hat{\sigma}$ is (most likely) close to $\sigma_0$, which is unrelated to the original order of the data. \scalebox{0.9}{$\sigma_0^{-1}$} therefore brings \scalebox{0.9}{$\sigma^*$} back to a shuffled version of the original sequence (identity permutation $\sigma_I$). Note that since Alg. 1 is publicly known, the adversary/analyst knows $\sigma_0$. Hence, even in the absence of this step from our algorithm, the adversary/analyst could perform this anyway. Finally, we permute $\by$ according to $\sigma^*$ and output the result \scalebox{0.9}{$\bz = \hat{\sigma}(\by)$} (Steps 9-10).  
%\vspace{-0.03cm} 
\begin{thm} Alg. 1 is $(\alpha,\calG)$-\name~private where $\alpha = \theta \cdot \Delta(\sigma_0 : \textswab{d}, \calG)$.
% \vspace{-0.25cm} 
\label{thm:privacy} 
\end{thm} 
The proof is in App. \ref{app:thm:privacy}.
Note that Alg.  1 provides the same level of privacy \scalebox{0.9}{$(\alpha)$} for any two group assignment \scalebox{0.9}{$\calG, \calG'$} as long as they have the same sensitivity, i.e, \scalebox{0.9}{$\Delta(\sigma_0 : \textswab{d}_\tau, \calG)=\Delta(\sigma_0 : \textswab{d}_\tau, \calG')$}. This leads to the following theorem which generalizes the privacy guarantee for any group assignment. 

\begin{thm} Alg. 1 satisfies 
$(\alpha',\calG')$-\name privacy for any group assignment $\calG'$ with $ \alpha'=\alpha\frac{\Delta(\sigma_0 : \textswab{d}, \calG')}{\Delta(\sigma_0 : \textswab{d}, \calG)}$ (proof in App. \ref{app:thm:generalized}.) %For instance, for Kendall $\tau$'s distance we have $\alpha'=\alpha\frac{\omega^{\sigma_0}_{\calG'}(\omega^{\sigma_0}_{\calG'}-1)}{\omega^{\sigma_0}_{\calG}(\omega^{\sigma_0}_{\calG}-1)}$. 
%  \vspace{-0.2cm}
\label{thm:generalized:privacy}
\end{thm} 
% Next, we present a utility theorem for Alg. 1 that formalizes the $(\eta,\delta)$-preservation for Hamming distance $\textswab{d}_{H}(\cdot)$  (we chose $\textswab{d}_{H}(\cdot)$ for the ease of numerical computation).
% \begin{thm} For a given set $S \subset [n]$ and Hamming distance metric,  $\textswab{d}_H(\cdot)$,   Alg. \ref{algo:main} is $(\eta,\delta)$-preserving for $\delta=\frac{1}{\psi(\theta, \textswab{d}_H)}\sum_{h=2k+1}^{n} (e^{-\theta\cdot h} \cdot c_h)$ where \scalebox{0.9}{$k=\lceil(1-\eta)\cdot |S|\rceil$} and $c_h$ is the number of permutations with hamming distance $h$ from the reference permutation that do not preserve \scalebox{0.9}{$\eta\%$} of $S$ (exact formula and proof in App. \ref{app:utility}).
% \label{thm:utility}
% \end{thm}
% \vspace{-0.25cm}

% The time complexity of Alg. 1 is dominated by \scalebox{0.9}{$\calG$}'s computation. Regardless of the method for selecting the reference permutation, the time complexity for the shuffling mechanism is at least \scalebox{0.9}{$\Omega(\sum_i|G_i|)$} (minimum computation for listing \scalebox{0.9}{$\calG$}). Note that Alg. \ref{algo:main}'s computation of the reference permutation takes \scalebox{0.9}{$O(|V|+|E|)=O(\sum_i|G_i|)$}. Another observation is that 

%The proof is in  App. \ref{app:thm:generalized}. %A utility theorem for Alg. 1 that formalizes the $(\eta,\delta)$- preservation for Hamming distance $\textswab{d}_{H}(\cdot)$  (we chose $\textswab{d}_{H}(\cdot)$ for the ease of numerical computation) is in App. \ref{app:utility:formal}. 

\textbf{Note.} Producing $\sigma^*$ is completely data ($\by$) independent. It only requires access to the public auxiliary information $\bt$. Hence, Steps $1-6$ can be performed in a pre-processing phase and do not contribute to the actual running time. See App. \ref{app:alg:illustration} for an illustration of Alg. 1 and runtime analysis. \nocite{RIM}
\newcommand{\calib}{\texttt{Cal}}
\vspace{-1.2em}
\section{Evaluation}\label{sec:eval}
\vspace{-1.5em}
\begin{figure*}[ht]
   \begin{subfigure}[b]{0.25\linewidth}
       \centering
       \includegraphics[width=0.95\linewidth]{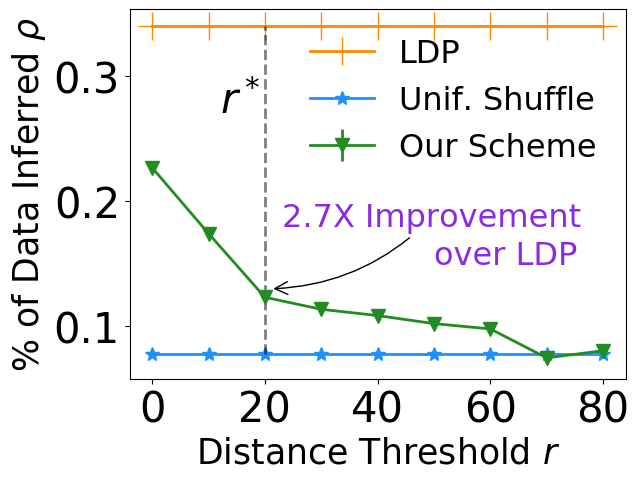}
    %  \vspace{-0.15cm}
       \caption{\textit{PUDF}: Attack}
       \label{fig:Texas:attack}
    \end{subfigure}%%
    % \begin{subfigure}[b]{0.24\linewidth}
    %     \centering
    %     \includegraphics[width=\linewidth]{./figures/Adult_attack_new.png}
    %     %\vspace{-0.15cm}
    %     \caption{\textit{Adult}: Attack }%($r$)}
    %     \label{fig:Adult:attack}
    % \end{subfigure}%%
    \begin{subfigure}[b]{0.25\linewidth}
        \centering
        \includegraphics[width=0.95\linewidth]{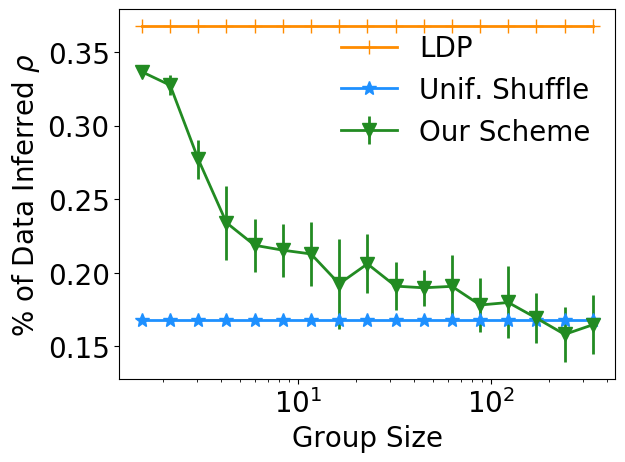}
        \caption{\textit{Twitch}: Attack}
        \label{fig:Twitch:attack}
    \end{subfigure}%%
    %   \begin{subfigure}[b]{0.24\linewidth}
    %     \centering
    %     \includegraphics[width=\linewidth]{./figures/Adult_alpha.png}
    %     %   \vspace{-0.15cm} 
    %     \caption{\textit{Adult}: Attack ($\alpha$)}
    %     \label{fig:Adult:alpha}
    % \end{subfigure}\\
    \begin{subfigure}[b]{0.25\linewidth}
        \centering
        \includegraphics[width=0.95\linewidth]{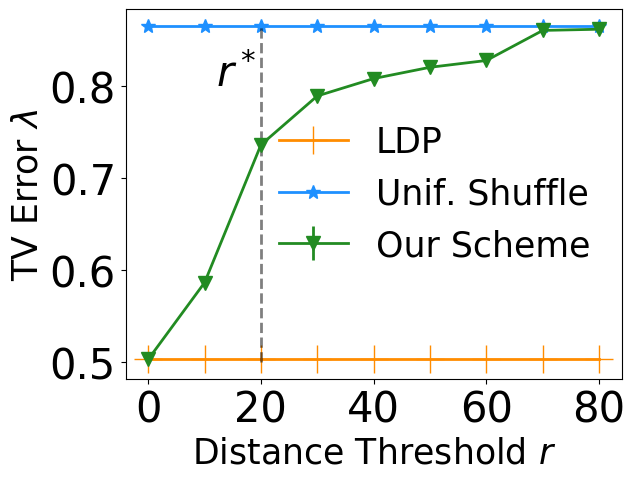}
        %\vspace{-0.15cm}
        \caption{\textit{PUDF}: Learnability}
        \label{fig:Texas:utility}
    \end{subfigure}%%
    % \begin{subfigure}[b]{0.24\linewidth}
    %     \centering
    %     \includegraphics[width=\linewidth]{./figures/Adult_utility_1.png}
    %     %\vspace{-0.15cm}
    %     \caption{\textit{Adult}: Learnability}
    %     \label{fig:Adult:utility}
    % \end{subfigure}
    \begin{subfigure}[b]{0.25\linewidth}
        \centering
        \includegraphics[width=0.95\linewidth]{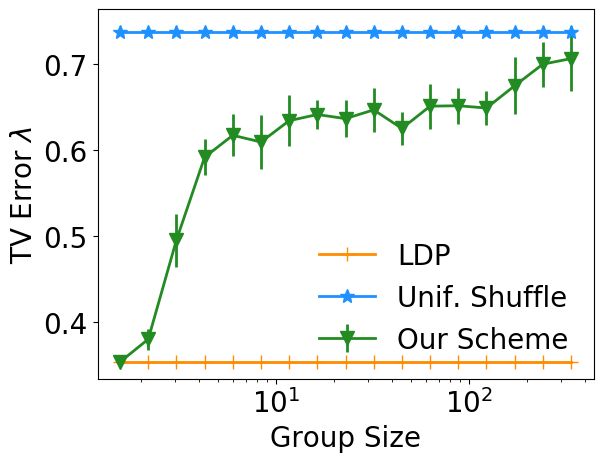}
        \caption{\textit{Twitch}: Learnability}
        \label{fig:Twitch:utility}
    \end{subfigure}
    % \begin{subfigure}[b]{0.24\linewidth}
    %     \centering
    %     \includegraphics[width = \linewidth]{./figures/Syn_utility.png}
    %     %\vspace{-0.15cm}
    %     \caption{\textit{Syn}: Learnability}% that acts as the public auxiliary information.}
    %     \label{fig:Syn:utility}
    % \end{subfigure}
    \vspace{-0.2cm}
   \caption{
   Our scheme interpolates between standard LDP (orange line) and uniform shuffling (blue line) in both privacy and data learnability. All plots increase group size along x-axis (except (d)). 
   (a) $\rightarrow$ (b): The fraction of participants vulnerable to an inferential attack.  
   %(d): Attack success with varying $\alpha$ for a fixed $r$. 
   (c) $\rightarrow$ (d): The accuracy of a calibration model trained on $\bz$ predicting the distribution of \ldp outputs at any point $t \in \calT$, such as the distribution of medical insurance types used specifically in the Houston area (not possible when uniformly shuffling across Texas). 
   %(h): Test accuracy of a classifier trained on $\bz$ for the synthetic dataset in Fig. \ref{fig:demonstration}. 
   }
   \label{fig:results}
   \vspace{-0.3cm}
\end{figure*}

The previous sections describe how our shuffling framework interpolates between standard \ldp and uniform random shuffling. We now experimentally evaluate this asking the following two questions -- 

\textbf{Q1.} Does the Alg. 1 mechanism protect against realistic inference attacks? \\\textbf{Q2.} How well can Alg. 1 tune a model's ability to learn trends within the shuffled data, i.e., tune \emph{data learnability}?

We evaluate on four datasets.  We are not aware of any prior work that provides comparable local inferential privacy. Hence, we baseline our mechanism with the two extremes: standard \textsf{LDP} and uniform random shuffling.  For concreteness, we detail our procedure with the \textit{PUDF} dataset~\citep{PUDF} \href{https://www.dshs.state.tx.us/THCIC/Hospitals/Download.shtm}{(license)}, which comprises $n \approx 29$k psychiatric patient records from Texas. Each data owner's sensitive value $x_i$ is their medical payment method, which is reflective of socioeconomic class (such as medicaid or charity). Public auxiliary information $t \in \calT$ is the hospital's geolocation. Such information is used for understanding how payment methods (and payment amounts) vary from town to town for insurances in practice \citep{insurance}. Uniform shuffling across Texas precludes such analyses. Standard \ldp risks inference attacks, since patients attending hospitals in the same neighborhood have similar socioeconomic standing and use similar payment methods, allowing an adversary to correlate their noisy $y_i$'s. To trade these off, we apply Alg. 1 with $d(\cdot)$ being distance (km) between hospitals, $\alpha = 4$ and Kendall's $\tau$ rank distance measure for permutations. %For each $r$, we recompute the grouping $\calG$, and produce a new shuffled $\bz$. 

Our inference attack predicts $\DO_i$'s $x_i$ by taking a majority vote of the $z_j$ values of the $25$ data owners within $r^*$ of $t_i$ and who are most similar to $\DO_i$ w.r.t some additional privileged auxiliary information $t^p_j \in \calT_p$. For PUDF, this includes the $25$ data owners who attended hospitals that are within $r^*$ km of $\DO_i$'s hospital, and are most similar in payment amount $t^p_j$. Using an \scalebox{1}{$\epsilon = 2.5$} randomized response mechanism, we resample the \ldp sequence $\by$ 50 times, and apply Alg. 1's chosen permutation to each, producing 50 $\bz$'s. We then mount the majority vote attack on each $x_i$ for each $\bz$. If the attack on a given $x_i$ is successful across $\geq 90\%$ of these \ldp trials, we mark that data owner as vulnerable -- although they randomize with \textsf{LDP}, there is a $\geq 90\%$ chance that a simple inference attack can recover their true value. We record the fraction of vulnerable data owners as $\rho$. %(an adversary could very feasibly know which data owners are vulnerable a priori). 
We report 1-standard deviation error bars over 10 trials.   

Additionally, we evaluate \emph{data learnability} --  how well the underlying statistics of the dataset are preserved across $\calT$. %how well a model trained on $\bz$ can infer the distribution of $x_i$'s at any point $t \in \calT$. 
For \textit{PUDF}, this means training a model on the shuffled $\bz$ to predict the distribution of payment methods used near, for instance, $t_i=$ Houston for $\DO_i$. For this, we train a calibrated model, \scalebox{1}{$\calib:\calT \rightarrow \mathcal{D}_x$}, on the shuffled outputs where \scalebox{1}{$\mathcal{D}_x$} is the set of all distributions on the domain of sensitive attributes \scalebox{1}{$\calX$}. We implement $\calib$ as a gradient boosted decision tree (GBDT) model \citep{gradientboosting} calibrated with Platt scaling \citep{calibration}. For each location $t_i$, we treat the empirical distribution of $x_i$ values within $r^*$ as the ground truth distribution at $t_i$, denoted by \scalebox{1}{$\mathcal{E}(t_i) \in \mathcal{D}_x$}. Then, for each $t_i$, we measure the Total Variation error between the predicted and ground truth distributions \scalebox{1}{$\text{TV}\big( \mathcal{E}(t_i), \calib_r(t_i) \big)$}. We then report $\lambda(r)$ -- the average TV error for distributions predicted at each $t_i \in \bt$ normalized by the TV error of naively guessing the uniform distribution at each $t_i$. With standard \textsf{LDP}, this task can be performed relatively well at the risk of inference attacks. With uniformly shuffled data, it is impossible to make geographically localized predictions unless the distribution of payment methods is identical in every Texas locale.

We additionally perform the above experiments on the following three datasets 
\squishlist \vspace{-0.25cm} 
% \item \textit{PUDF}.  The payment method (which can be reflective of social class) is considered private $\calX = \{$ medicare, medicaid, charity, private$\}$. $\calT$ is the hospital zipcode and $\calT_P$ is the charge amount. In fact, this correlation is used for medical insurances in practice \cite{insurance}. \vspace{-0.05cm}
\item \textit{\href{http://snap.stanford.edu/data/twitch-social-networks.html}{Twitch}}~\citep{twitch}. This dataset, gathered from the \emph{Twitch} social media platform, includes a graph of $\approx 9K$ edges (mutual friendships) along with node features. The user's history of explicit language is private $\calX = \{0,1\}$. $\calT$ is a user's mutual friendships, i.e. $t_i$ is the $i$'th row of the graph's adjacency matrix. We do not have any $\calT_P$ here and select the 25 neighbors randomly. 
\item \textit{Syn}. This is a synthetic dataset of size $20K$ which can be classified at three granularities -- 8-way, 4-way and 2-way  (Fig. \ref{fig:data} shows a scaled down version of the dataset). The eight color labels are private $\calX = [8]$; the 2D-positions are public $\calT = \R^2$. For learnability, we measure the accuracy of $8$-way, $4$-way and $2$-way GBDT models trained on $\bz$ on an equal sized test set at each $r$.
\item \textit{\href{https://archive.ics.uci.edu/ml/datasets/Adult}{Adult}}~\citep{adult}. This dataset is derived from the 1994 Census and has $\approx33$K records.  Whether $\DO_i$'s annual income is $\geq 50$k is considered private, $\calX = \{\geq 50k, <50k\}$. $\calT = [17, 90]$ is age and $\calT_P$ is the  individual's marriage status. Due to lack of space figures are in App. \ref{app:adult experiments}. %\textcolor{blue}{Trends observed on Adult reinforce those of the three datasets included here, see  for figures.}
% \vspace{-0.05cm}
% \vspace{-0.1cm}
\squishend 

%\vspace{-0.3cm}
\subsection{Experimental Results}
%The plots of Figure \ref{fig:results} indicate that our scheme successfully interpolates between standard \ldp and uniform shuffling. We return to our two primary questions: 
\textbf{Q1.} Our formal guarantee on the inferential privacy loss (Thm. \ref{thm: semantic guarantee}) is described w.r.t to a `strong' adversary (with access to  \scalebox{1}{$\{y_{G_i}\},\by_{\overline{G}_i}$}). Here, we test how well does our proposed scheme (Alg. 1) protect against inference attacks on real-world datasets without any such assumptions. Additionally, to make our attack more realistic, the adversary has access to extra privileged auxiliary information $\calT_P$ which is \textit{not used} by Alg. \ref{algo:main}.  Fig. \ref{fig:Texas:attack}$\rightarrow$ \ref{fig:Twitch:attack} show  that our scheme significantly reduces the attack efficacy. For instance, $\rho$ is reduced by \scalebox{1}{$2.7X$} at the attack distance threshold $r^*$ for \textit{PUDF}.  \begin{wrapfigure}{r}{0.3\linewidth}
     \vspace{-0.4cm}
    \hspace{-0.8cm}
    \centering
    \includegraphics[height=3cm]{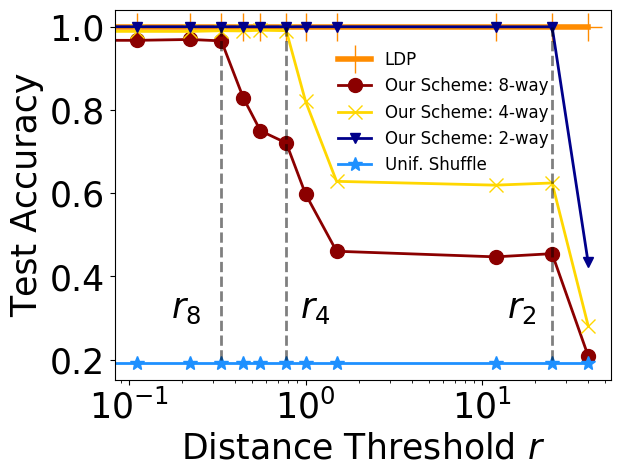} 
    \vspace{-0.3cm}
    \caption{\textit{Syn}: Learnability}% that acts as the public auxiliary information.}
    \label{fig:Syn:utility}
    \vspace{-1.5em}
\end{wrapfigure} Additionally, $\rho$ for our scheme varies from that of \textsf{LDP}\footnote{Our scheme gives lower $\rho$ than \ldp at \scalebox{1}{$r=0$} because the resulting groups are non-singletons. For instance, for PUDF, $G_i$ includes all individuals with the same zipcode as $\DO_i$.} (minimum privacy)   to uniform shuffle (maximum privacy) with increasing $r$ (equivalently group size as in Fig. \ref{fig:Twitch:attack}) thereby spanning the entire privacy spectrum. As expected, $\rho$ decreases with decreasing privacy parameter $\alpha$ (Fig. \ref{fig:Adult:alpha}).
\\\textbf{Q2.} Fig.\ref{fig:Texas:utility} \scalebox{1}{$\rightarrow$} \ref{fig:Twitch:utility} show that $\lambda$ varies from that of \ldp (maximum learnability) to that of uniform shuffle (minimum learnability) with increasing $r$ (equivalently, group size), thereby providing tunability.  Interestingly, for \textit{Adult} our scheme reduces $\rho$ by \scalebox{1}{$1.7X$} at the same $\lambda$ as that of \ldp for \scalebox{1}{$r=1$} (Fig. \ref{fig:Adult:utility}). Fig. \ref{fig:Syn:utility} shows that the distance threshold $r$ defines the granularity at which the data can be classified. \ldp allows 8-way classification while uniform shuffling allows none. The granularity of classification can be tuned by our scheme -- $r_8$, $r_4$ and $r_2$ mark the thresholds for $8$-way, $4$-way and $2$-way classifications, respectively. %Experiments on evaluation of \scalebox{1}{$(\eta,\delta)$}-preservation are in  App. \ref{app:extraresults}.

% \vspace{-0.5cm}
\section{Conclusion}
%\vspace{-0.5cm}
We have proposed a new privacy definition, \name-privacy that casts new light on the inferential privacy benefits of shuffling and a novel shuffling mechanism to achieve the same. %Additionally, we propose a generalized shuffle framework that interpolates between \ldp and uniform shuffling in terms of the protection afforeded against inference attacks and data learnability.

\newpage

% \section{Ethis Statement}
% It is the aim of this paper to formalize the enhanced privacy guarantees offered by shuffling, to provide intuition of what those formal guarantees semantically offer to data owners, and to demonstrate an algorithm + experiments which offer these guarantees while meeting analyst utility requirements. We feel that all of these aims as well as the public datasets used are ethical. 

% \section{Reproducibility Statement}
% The majority of this paper formalizes a novel perspective on the privacy guarantees achieved by shuffling (i.e. randomizing the order of the data as opposed to the values). Detailed proofs as well as intuitive discussions are provided in the Appendix. All datasets are public. A \texttt{.zip} file demonstrating code of each experiment has been uploaded as supplementary material. 

\bibliography{refs}
\bibliographystyle{iclr2022_conference.bst}

\appendix

% \twocolumn
% [\textbf{Local Inferential Privacy through Data Shuffling -- Supplementary Material }]

\section{Appendix}\label{app}
\subsection{Background Cntd.}\label{app:background}

\subsection{Local Inferential Privacy}  \vspace{-0.2cm}
%introduce pufferfish inferntial log loss 
%In this section, we introduce some context for inferential privacy in the \ldp setting. 
%Inferential privacy captures the privacy loss in the face of an informed adversary in a Bayesian framework. 
Local inferential privacy captures what information a Bayesian adversary \cite{Pufferfish}, with some prior, can learn in the \ldp setting. 
Specifically, it measures the largest possible ratio between the adversary's posterior and prior beliefs about an individual’s data after observing a mechanism's output .%\footnote{This quantity is identical to the \ldp parameter of the mechanism when\textit{individuals’ data are independent}\cite{sok,}.}.
\begin{defn}(Local Inferential Privacy Loss \cite{Pufferfish}) Let $\bx=\langle x_1, \cdots, x_n\rangle$ and let $\by=\langle y_1, \cdots, y_n \rangle$ denote the input (private) and output sequences (observable to the adversary) in the \ldp setting. Additionally, the adversary's auxiliary knowledge is modeled by a prior distribution $\mathcal{P}$ on $\mathbf{x}$. The inferential privacy loss for the input sequence $\mathbf{x}$ is given by
% \vspace{0cm} 
\begin{equation}
% \vspace{-0.1cm}
\small \mathbb{L}_{\calP}(\mathbf{x})=\max_{\substack{i\in [n]\\ a,b \in \calX}}\Bigg(\log\frac{\mathrm{Pr}_{\calP}[\mathbf{y}|x_i=a]}{\mathrm{Pr}_{\calP}[\mathbf{y}|x_i=b]}\Bigg)
=\max_{\substack{i\in [n]\\ a,b \in \calX}}\Bigg (	\bigg| \log \frac{\mathrm{Pr}_{\calP}[x_i = a | \bf{y} ]}{\mathrm{Pr}_{\calP}[x_i = b | \bf{y}]}
	- \log \frac{\mathrm{Pr}_{\calP}[x_i = a]}{\mathrm{Pr}_{\calP}[x_i = b]} \bigg|\Bigg)
\end{equation}
\label{def:ip}
\vspace{-1em}
\end{defn}
% Using Bayes' theorem, we have\vspace{-0.2cm}
% \begin{gather*}\small\vspace{-0.4cm} \mathbb{L}_{\calP}(\mathbf{x})=\max_{\substack{i\in [n]\\ a,b \in \calX}}\Bigg (	\bigg| \log \frac{\mathrm{Pr}_{\calP}[x_i = a | \bf{y} ]}{\mathrm{Pr}_{\calP}[x_i = b | \bf{y}]}
% 	- \log \frac{\mathrm{Pr}_{\calP}[x_i = a]}{\mathrm{Pr}_{\calP}[x_i = b]} \bigg|\Bigg)\vspace{-0.7cm}\end{gather*}
Bounding  $\mathbb{L}_{\calP}(\mathbf{x})$  would imply
 that the adversary's belief about the value of any $x_i$ does not change by much even after observing the output sequence $\bf{y}$. This means that an informed adversary does not learn much about the individual $i$'s private input upon observation of the entire private dataset $\by$.

Here we define two rank distance measures \begin{defn}[Kendall's $\tau$ Distance] For any two permutations, $\sigma, \pi \in \mathrm{S}_n$, the Kendall's $\tau$
distance $\textswab{d}_{\tau}(\sigma, \pi)$ counts the number of pairwise disagreements between $\sigma$ and $\pi$, i.e., the
number of item pairs that have a relative order in one permutation and a different order in
the other. Formally, \begin{gather*}\textswab{d}_{\tau}(\sigma,\pi)=\Big| \ \big\{(i,j) : i < j,  \big[\sigma(i) > \sigma(j) \wedge \pi(i) < \pi(j) \big]\\\hspace{2cm}
		\vee \big[\sigma(i) < \sigma(j) \wedge \pi(i) > \pi(j)\big] \big\} \ \Big|\numberthis\label{eq:kendalltau}\end{gather*}  \label{def:kendall} \end{defn}
%Equivalently, $d_{\tau}(\sigma, \pi)$ is defined as the number of adjacent swaps to convert$\sigma^{-1}$into $\pi^{-1}$.
For example, if $\sigma=(1 \:\: 2 \: \: 3 \: \: 4 \:  \: 5 \: \: 6 \: \: 7 \: \: 8 \: 9 \: 10)$ and  $\pi=(1\:2\:3\:\underline{6} \: 5 \: \underline{4}\:7\:8\:9\:10)$, then $\textswab{d}_{\tau}(\sigma,\pi)=3$.

 Next, Hamming distance measure is defined as follows.
 
\begin{defn}[Hamming Distance] 
For any two permutations, $\sigma, \pi \in \mathrm{S}_n$, the Hamming distance $\textswab{d}_{H}(\sigma, \pi)$ counts the number of positions in which the two permutations disagree. Formally, 
\begin{align*}
    \textswab{d}_H(\sigma, \pi)
    &= \Big| \big\{ i \in [n] : \sigma(i) \neq \pi(i) \big\} \Big| 
\end{align*}
Repeating the above example, if $\sigma=(1 \:\: 2 \: \: 3 \: \: 4 \:  \: 5 \: \: 6 \: \: 7 \: \: 8 \: 9 \: 10)$ and  $\pi=(1 \: 2 \: 3 \: \underline{6} \: 5 \: \underline{4} \: 7 \: 8 \: 9 \: 10)$, then $\textswab{d}_{H}(\sigma,\pi)=2$.
\end{defn}

\subsection{\name-privacy and the De Finetti attack}
\label{app:de finetti}
We now show that a strict instance of \name privacy is sufficient for thwarting any de Finetti attack \cite{definetti} on individuals. The de Finetti attack involves a Bayesian adversary, who, assuming some degree of correlation between data owners, attempts to recover the true permutation from the shuffled data. As written, the de Finetti attack assumes the sequence of sensitive attributes and side information $(x_1, t_1), \dots, (x_n, t_n)$ are \emph{exchangeable}: any ordering of them is equally likely. By the de Finetti theorem, this implies that they are i.i.d. conditioned on some latent measure $\theta$. To balance privacy with utility, the $\bx$ sequence is non-uniformly randomly shuffled w.r.t. the $\bt$ sequence producing a shuffled sequence $\bz$, which the adversary observes. Conditioning on $\bz$ the adversary updates their posterior on $\theta$ (i.e. posterior on a model predicting $x_i | t_i$), and thereby their posterior predictive on the true $\bx$. The definition of privacy in \cite{definetti} holds that the adversary's posterior beliefs are close to their prior beliefs by some metric on distributions in $\calX$, $\delta(\cdot, \cdot)$: 
\begin{align*}
    \delta\Big( \Pr[x_i], \Pr[x_i | \bz] \Big) \leq \alpha 
\end{align*}

We now translate the de Finetti attack to our setting. First, to align notation with the rest of the paper we provide privacy to the sequence of $\ldp$ values $\by$ since we shuffle those instead of the $\bx$ values as in \cite{definetti}. We use max divergence (multiplicative bound on events used in \DP) for $\delta$: 
\begin{align*}
    \Pr[y_i \in O] &\leq e^\alpha \Pr[y_i \in O | \bz] \\
    \Pr[y_i \in O | \bz] &\leq e^\alpha \Pr[y_i \in O]
\end{align*}
which, for compactness, we write as 
\begin{align}
    \Pr[y_i \in O] \approx_\alpha \Pr[y_i \in O | \bz] \quad. 
    \label{eq:definetti privacy}
\end{align}
We restrict ourselves to shuffling mechanisms, where we only randomize the order of sensitive values. By learning the unordered values $\{y\}$ alone, an adversary may have arbitrarily large updates to its posterior (e.g. if all values are identical), breaking the privacy requirement above. With this in mind, we assume the adversary already knows the unordered sequence of values $\{y\}$ (which they will learn anyway), and has a prior on permutations $\sigma$ allocating values from that sequence to individuals. We then generalize the de Finetti problem to an adversary with an \emph{arbitrary} prior on the true permutation $\sigma$, and observes a randomize permutation $\sigma'$ from the shuffling mechanism. We require that the adversary's prior belief that $\sigma(i) = j$ is close to their posterior belief for all $i,j \in [n]$: 
\begin{align}
    \Pr[\sigma \in \Sigma_{i,j} ] \approx_\alpha \Pr[\sigma \in \Sigma_{i,j} | \sigma'] \quad \forall i,j \in [n], \forall \sigma' \in S_n \quad ,
    \label{eq:definetti privacy II}
\end{align}
where $\Sigma_{i,j} = \{\sigma \in S_n : \sigma(i) = j\}$, the set of permutations assigning element $j$ to $\DO_i$. Conditioning on any unordered sequence $\{y\}$ with all unique values, the above condition is necessary to satisfy Eq. \eqref{eq:definetti privacy} for events of the form $O = \{y_i = a\}$, since $\{y_i = a\} = \{\Sigma_{i,j}\}$ for some $j \in [n]$. For any $\{y\}$ with repeat values, it is sufficient since $\Pr[y_i = a]$ is the sum of probabilities of disjoint events of the form $\Pr[\sigma \in \Sigma_{i,k}]$ for various $k \in [n]$ values. 

We now show that a strict instance of \name-privacy satisfies Eq. \eqref{eq:definetti privacy II}. Let $\widehat{\calG}$ be any group assignment such that at least one $G_i \in \widehat{\calG}$ includes all data owners, $G_i = \{1, 2, \dots, n\}$. 

\begin{prope}
A $(\widehat{\calG}, \alpha)$-\name-private shuffling mechanism $\sigma' \sim \calA$ satisfies  
\begin{align*}
    \Pr[\sigma \in \Sigma_{i,j} ] \approx_\alpha \Pr[\sigma \in \Sigma_{i,j} | \sigma']
\end{align*}
for all $i,j \in [n]$ and all priors on permutations $\Pr[\sigma]$. 
\end{prope}

\begin{proof}

\begin{lemma}
\label{lem:definetti equivalent}
    For any prior $\Pr[\sigma]$, Eq. \eqref{eq:definetti privacy II} is equivalent to the condition
    \begin{align}
        \frac{\sum_{\hat{\sigma} \in \overline{\Sigma}_{i,j}} \Pr[\hat{\sigma}] \Pr[\sigma' | \hat{\sigma}] }{
        \sum_{\hat{\sigma} \in {\Sigma_{i,j}}} \Pr[\hat{\sigma}] \Pr[\sigma' | \hat{\sigma}] } 
        \approx_\alpha 
        \frac{\sum_{\hat{\sigma} \in \overline{\Sigma}_{i,j}} \Pr[\hat{\sigma}] }{
        \sum_{\hat{\sigma} \in {\Sigma_{i,j}}} \Pr[\hat{\sigma}] }
        \label{eq:definetti privacy III}
    \end{align}
    where the set $\overline{\Sigma}_{i,j}$ is the complement of ${\Sigma}_{i,j}$. 
\end{lemma}
Under grouping $\hat{\calG}$, every permutation $\sigma_a \in {\Sigma}_{i,j}$ neighbors every permutation $\sigma_b \in \overline{\Sigma}_{i,j}$, $\sigma_a \approx_{\hat{\calG}} \sigma_b$, for any $i,j$. By the definition of \name-privacy, we have that for any observed permutation $\sigma'$ output by the mechanism: 
\begin{align*}
    \Pr[\sigma' | \sigma = \sigma_a] \approx_\alpha \Pr[\sigma' | \sigma = \sigma_b]
    \quad \forall \sigma_a \in {\Sigma}_{i,j}, \sigma_b \in \overline{\Sigma}_{i,j}, \sigma' \in S_n 
    \quad .
\end{align*}
This implies Eq. \ref{eq:definetti privacy III}. Thus, $(\widehat{\calG}, \alpha)$-\name-privacy implies Eq. \ref{eq:definetti privacy III}, which implies Eq. \ref{eq:definetti privacy II}, thus proving the property. 
\end{proof}

Using Lemma \ref{lem:definetti equivalent}, we may also show that this strict instance of \name-privacy is \emph{necessary} to block all de Finetti attacks: 

\begin{prope}
A $(\widehat{\calG}, \alpha)$-\name-private shuffling mechanism $\sigma' \sim \calA$ is necessary to satisfy 
\begin{align*}
    \Pr[\sigma \in \Sigma_{i,j} ] \approx_\alpha \Pr[\sigma \in \Sigma_{i,j} | \sigma']
\end{align*}
for all $i,j \in [n]$ and all priors on permutations $\Pr[\sigma]$. 
\end{prope}

\begin{proof}
If our mechanism $\calA$ is not $(\widehat{\calG}, \alpha)$-\name-private, then for some pair of true (input) permutations $\sigma_a \neq \sigma_b$ and some released permutation $\sigma' \sim \calA$, we have that 
\begin{align*}
    \Pr[\sigma' | \sigma_b] \geq e^\alpha \Pr[\sigma' | \sigma_a]\quad. 
\end{align*}
Under $\hat{\calG}$, all permutations neighbor each other, so $\sigma_a \approx_{\hat{\calG}} \sigma_b$. Since $\sigma_a \neq \sigma_b$, then for some $i,j \in [n]$, $\sigma_a \in \Sigma_{i,j}$ and $\sigma_b \in \overline{\Sigma}_{i,j}$: one of the two permutations assigns some $j$ to some $\DO_i$ and the other does not. Given this, we may construct a bimodal prior on the true $\sigma$ that assigns half its probability mass to $\sigma_a$ and the rest to $\sigma_b$, 
\begin{align*}
    \Pr[\sigma_a] = \Pr[\sigma_b] = \frac{1}{2} \quad .
\end{align*}
Therefore, for released permutation $\sigma'$, the RHS of Eq. \ref{eq:definetti privacy III} is 1, and the LHS is 
\begin{align*}
    \frac{\sum_{\hat{\sigma} \in \overline{\Sigma}_{i,j}} \Pr[\hat{\sigma}] \Pr[\sigma' | \hat{\sigma}] }{
        \sum_{\hat{\sigma} \in {\Sigma_{i,j}}} \Pr[\hat{\sigma}] \Pr[\sigma' | \hat{\sigma}] }
        &= \frac{\nicefrac{1}{2} \Pr[\sigma' | \sigma_b]}{\nicefrac{1}{2} \Pr[\sigma' | \sigma_a]} \\
        &\geq e^\alpha \quad , 
\end{align*}
violating Eq. \ref{eq:definetti privacy III}, thus violating Eq. \ref{eq:definetti privacy II}, and failing to prevent de Finetti attacks against this bimodal prior. 
\end{proof}

Ultimately, unless we satisfy \name-privacy shuffling the entire dataset, there exists some prior on the true permutation $\Pr[\sigma]$ such that after observing the shuffled $\bz$ permuted by $\sigma'$, the adversary's posterior belief on one permutation is larger than their prior belief by a factor $\geq e^\alpha$. If we suppose that the set of values $\{y\}$ are all distinct, this means that for some $a \in \{y\}$, the adversary's belief that $y_i = a$ is signficantly larger after observing $\bz$ than it was before. 

Now to prove Lemma \ref{lem:definetti equivalent}: 
\begin{proof}
\begin{align*}
    \Pr[\sigma \in \Sigma_{i,j} ] 
    &\approx_\alpha \Pr[\sigma \in \Sigma_{i,j} | \sigma'] \\
    \Pr[\sigma \in \Sigma_{i,j} ]
    &\approx_\alpha \frac{\Pr[\sigma' | \sigma \in \Sigma_{i,j}] \Pr[\sigma \in \Sigma_{i,j} ]}{\sum_{\hat{\sigma} \in S_n} \Pr[\hat{\sigma}] \Pr[\sigma' | \hat{\sigma}]} \\
    \sum_{\hat{\sigma} \in S_n} \Pr[\hat{\sigma}] \Pr[\sigma' | \hat{\sigma}] 
    &\approx_\alpha \Pr[\sigma' | \sigma \in \Sigma_{i,j}] \\
    \sum_{\hat{\sigma} \in S_n} \Pr[\hat{\sigma}] \Pr[\sigma' | \hat{\sigma}] 
    &\approx_\alpha \Pr[\sigma \in \Sigma_{i,j}]^{-1} \sum_{\hat{\sigma} \in \Sigma_{i,j}} \Pr[\hat{\sigma}] \Pr[\sigma' | \hat{\sigma}] \\
    \sum_{\hat{\sigma} \in \Sigma_{i,j}} \Pr[\hat{\sigma}] \Pr[\sigma' | \hat{\sigma}] +
    \sum_{\hat{\sigma} \in \overline{\Sigma}_{i,j}} \Pr[\hat{\sigma}] \Pr[\sigma' | \hat{\sigma}] 
    &\approx_\alpha \Pr[\sigma \in \Sigma_{i,j}]^{-1} \sum_{\hat{\sigma} \in \Sigma_{i,j}} \Pr[\hat{\sigma}] \Pr[\sigma' | \hat{\sigma}] \\
    \sum_{\hat{\sigma} \in \overline{\Sigma}_{i,j}} \Pr[\hat{\sigma}] \Pr[\sigma' | \hat{\sigma}] 
    &\approx_\alpha \sum_{\hat{\sigma} \in \Sigma_{i,j}} \Pr[\hat{\sigma}] \Pr[\sigma' | \hat{\sigma}] 
    \Big( \frac{1}{\Pr[\sigma \in \Sigma_{i,j}]} - 1 \Big)  \\
    \frac{\sum_{\hat{\sigma} \in \overline{\Sigma}_{i,j}} \Pr[\hat{\sigma}] \Pr[\sigma' | \hat{\sigma}] }{
    \sum_{\hat{\sigma} \in {\Sigma_{i,j}}} \Pr[\hat{\sigma}] \Pr[\sigma' | \hat{\sigma}] } 
    &\approx_\alpha 
    \frac{\sum_{\hat{\sigma} \in \overline{\Sigma}_{i,j}} \Pr[\hat{\sigma}] }{
    \sum_{\hat{\sigma} \in {\Sigma_{i,j}}} \Pr[\hat{\sigma}] }
\end{align*}
\end{proof}

As such, a strict instance of \name-privacy can defend against any de Finetti attack (i.e. for any prior $\Pr[\sigma]$ on permutations), wherein at least one group $G_i \in \calG$ includes all data owners. Furthermore, it is necessary. This makes sense. In order to defend against any prior, we need to significantly shuffle the entire dataset. Without a restriction of priors as in Pufferfish \cite{Pufferfish}, the de Finetti attack (i.e. uninformed Bayesian adversaries) is an indelicate metric for evaluating the privacy of shuffling mechanisms: to achieve significant privacy, we must sacrifice all utility. This in many regards is reminiscent of the no free lunch for privacy theorem established in \cite{Kifer}. As such, there is a need for more flexible privacy definitions for shuffling mechanisms.

\subsection{ Additional Properties of \name-privacy} \label{app:post-processing}

\begin{lemma}[Convexity] \label{lem:convexity}
Let $\calA_1, \dots \calA_k: \mathcal{Y}^n \mapsto \mathcal{V}$ be a collection of $k$ $(\alpha,\calG)$-\name private mechanisms for a given group assignment $\calG$ on a set of $n$ entities. Let $\calA: \mathcal{Y}^n \mapsto \mathcal{V}$ be a convex combination of these $k$ mechanisms, where the probability of releasing the output of mechanism $\calA_i$ is $p_i$, and $\sum_{i=1}^k p_i = 1$. $\calA$ is also $(\alpha,\calG)$-\name private w.r.t. $\calG$. 
\end{lemma}
\begin{proof}
For any $(\sigma, \sigma') \in \mathrm{N}_\calG$ and $\by \in \calY$: 
\begin{align*}
    \mathrm{Pr} [\calA \big( \sigma( \by) \big) \in O]
    &= \sum_{i=1}^k p_i \mathrm{Pr} [\calA_i \big( \sigma( \by) \big) \in O] \\
    &\leq e^\alpha \sum_{i=1}^k p_i \mathrm{Pr} [\calA_i \big( \sigma'( \by) \big) \in O] \\
    &=  \mathrm{Pr} [\calA \big( \sigma'( \by) \big) \in O]
\end{align*}
\end{proof}

% For a given group assignment $\calG$ on a set of $n$ entities and a privacy parameter $\alpha \in \R_{\geq0}$, a randomized  mechanism $\calA: \mathcal{Y}^n \mapsto \mathcal{V} $ is $(\alpha,\mathcal{G})$-\name~private if for all $\mathbf{y} \in \mathcal{Y}^n$ and neighboring permutations $\sigma, \sigma' \in \mathrm{N}_\calG$ and any subset of output $O\subseteq \mathcal{V}$, we have\vspace{-0.2cm} 
% \begin{gather*} \vspace{-0.5cm}
%     \mathrm{Pr}[\calA\big(\sigma(\mathbf{y})\big) \in O] \leq e^\alpha \cdot \mathrm{Pr}\big[\calA\big(\sigma'(\mathbf{y})\big) \in O \big] \numberthis \label{eq:privacy} \vspace{-0.2cm}
% \end{gather*}
% %where $\bz=\pi(\by)=\langle y_{\pi(1)},\cdots, y_{\pi(n)}\rangle, \pi \in \mathrm{S}_n$
%  $\sigma(\mathbf{y})$ and $\sigma'(\mathbf{y})$  are defined to be \textit{neighboring sequences}. 

\begin{thm}[Post-processing]\label{theorem:post}
Let \scalebox{0.9}{$\mathcal{A}: \mathcal{Y}^n \mapsto \mathcal{V}$} be  $(\alpha,\calG)$-\name private for a given group assignment $\calG$ on a set of $n$ entities. Let \scalebox{0.9}{$f : \mathcal{V} \mapsto \mathcal{V}'$} be an
arbitrary randomized mapping. Then \scalebox{0.9}{$f \circ \mathcal{A} : \mathcal{Y}^n \mapsto \mathcal{V}'$} is also $(\alpha,\calG)$-\name private. \vspace{-0.2cm}\end{thm}

\begin{proof}
Let $g: \mathcal{V}\rightarrow \mathcal{V}'$ be a deterministic, measurable function. For any output event $\mathcal{Z}\subset \mathcal{V}'$, let $\mathcal{W}$ be its preimage: \newline $\mathcal{W}=\{v \in \mathcal{V}| g(v) \in \mathcal{Z}\}$. Then, for any $(\sigma, \sigma') \in \mathrm{N}_\calG$,
\begin{align*}
    \mathrm{Pr}\Big[g\Big(\calA\big(\sigma(\by)\big)\Big)\in \mathcal{Z}\Big] 
    &= \mathrm{Pr}\Big[\calA\big(\sigma(\by)\big)\in \mathcal{W}\Big] \\ 
    &\leq e^{\alpha} \cdot \mathrm{Pr}\Big[\calA\big(\sigma'(\by)\big)\in \mathcal{W}\Big]\\ 
    &=e^{\alpha}\cdot \mathrm{Pr}\Big[g\Big(\calA\big(\sigma'(\by)\big)\Big)\in \mathcal{Z}\Big] 
\end{align*}
This concludes our proof because any randomized mapping
can be decomposed into a convex combination of measurable, deterministic functions \cite{Dwork}, and as Lemma \ref{lem:convexity} shows, a convex combination of $(\alpha,\calG)$-\name private mechanisms is also $(\alpha,\calG)$-\name private. 
\end{proof}

\begin{thm}[Sequential Composition] \label{theorem:seq}
If $\calA_1$ and $\calA_2$ are $(\alpha_1, \calG)$- and $(\alpha_2, \calG)$-\name private mechanisms, respectively, that use independent randomness, then releasing the outputs $\big( \calA_1(\by), \calA_2(\by) \big)$ satisfies $(\alpha_1+\alpha_2, \calG)$-\name privacy. 
\end{thm}
\begin{proof}
We have that $\calA_1: \calY^n \rightarrow \mathcal{V}'$ and $\calA_1: \calY^n \rightarrow \mathcal{V}''$ each satisfy \name-privacy for different $\alpha$ values. Let $\calA:\calY^n \rightarrow (\mathcal{V}' \times \mathcal{V}'')$ output $\big(\calA_1(\by), \calA_2(\by)\big)$. Then, we may write any event $\mathcal{Z} \in (\mathcal{V}' \times \mathcal{V}'')$ as $\mathcal{Z}' \times \mathcal{Z}''$, where $\mathcal{Z}' \in \mathcal{V}'$ and $\mathcal{Z}'' \in \mathcal{V}''$. We have for any $(\sigma, \sigma') \in \mathrm{N}_\calG$, 
\begin{align*}
    \mathrm{Pr} \big[ \calA\big( \sigma(\by) \big)  &\in \mathcal{Z} \big]  
    = \mathrm{Pr} \big[ \big(\calA_1\big( \sigma(\by) \big) , \calA_2\big( \sigma(\by) \big) \big) \in \mathcal{Z} \big] \\
    &= \mathrm{Pr} \big[ \{\calA_1\big( \sigma(\by) \big)  \in \mathcal{Z}'\} \cap \{\calA_2\big( \sigma(\by) \big)   \in \mathcal{Z}'' \} \big]  \\
    &= \mathrm{Pr} \big[ \{\calA_1\big( \sigma(\by) \big)  \in \mathcal{Z}'\} \big] 
    \mathrm{Pr} \big[ \{\calA_2\big( \sigma(\by) \big)   \in \mathcal{Z}'' \} \big] \\
    &\leq e^{\alpha_1 + \alpha_2}
    \mathrm{Pr} \big[ \{\calA_1\big( \sigma'(\by) \big)  \in \mathcal{Z}'\} \big] 
    \mathrm{Pr} \big[ \{\calA_2\big( \sigma'(\by) \big)   \in \mathcal{Z}'' \} \big] \\
    &= e^{\alpha_1 + \alpha_2} \cdot 
    \mathrm{Pr} \big[ \calA\big( \sigma'(\by) \big)  \in \mathcal{Z} \big] 
\end{align*}
\end{proof}

% Proof of Lemma \ref{lemma:LDP} \\

% \textbf{Lemma \ref{lemma:LDP}} \emph{
% An $\epsilon$-\ldp mechanism is $(k\epsilon, \calG)$-\name~ private for any group assignment $\calG$ such that $
%         k \geq \max_{G_i \in \calG} |G_i|
% $
% }
% \begin{proof}
% This follows from $k$-group privacy \cite{Dwork}. $\by$ are $\varepsilon$-LDP outputs $\calA_{\text{LDP}}(\bx)$ from input sequence $\bx$. For any $\sigma \approx_{G_i} \sigma'$, we know by definition that $\sigma(j) = \sigma'(j)$ for all $j \notin G_i$. As such, the permuted sequences $\sigma(\bx)_j = \sigma'(\bx)_j$ for all $j \notin G_i$, and differ in at most $|G_i|$ entries. In other words, 
% \begin{align*}
%     \textswab{d}_H \big(\sigma(\bx), \sigma'(\bx)\big) \leq |G_i| 
% \end{align*}
% Using this fact, we have from the $k$-group property of LDP that 
% \begin{align*}
%     \mathrm{Pr}\big[ \calA_{\text{LDP}} \big( \sigma(\bx) \big) \in O \big] 
%     \leq e^{|G_i| \epsilon} \mathrm{Pr}\big[ \calA_{\text{LDP}} \big( \sigma'(\bx) \big) \in O \big] 
% \end{align*}
% and thus if $k \geq \max_{G_i \in \calG} |G_i|$, 
% \begin{align*}
%     \mathrm{Pr}\big[ \calA_{\text{LDP}} \big( \sigma(\bx) \big) \in O \big] 
%     \leq e^{k \epsilon} \mathrm{Pr}\big[ \calA_{\text{LDP}} \big( \sigma'(\bx) \big) \in O \big] 
% \end{align*}
% for all $(\sigma, \sigma') \in \mathrm{N}_\calG$. 
% \end{proof}

\subsection{Proof for Thm. \ref{thm: semantic guarantee}}\label{app:thm:semantic}
\label{app:bayesian proof} 
\textbf{Theorem \ref{thm: semantic guarantee}} 
\emph{
For a given group assignment $\calG$ on a set of $n$ data owners, if a shuffling mechanism $\calA:\calY^n\mapsto \calY^n$ is $(\alpha,\calG)$-\name private, then for each data owner $\DO_i, i \in [n]$, %\vspace{-0.1cm}
\begin{align*}
   \max_{\substack{i\in [n]\\ a,b \in \calX}} \bigg|\log \frac{\Pr_\calP [x_i = a | \bz, \{y_{G_i}\},\by_{\overline{G}_i}]}{\Pr_\calP [x_i = b | \bz, \{y_{G_i}\},\by_{\overline{G}_i}]} - \log \frac{\Pr_\calP [x_i = a | \{y_{G_i}\},\by_{\overline{G}_i}]}{\Pr_\calP [x_i = b | \{y_{G_i}\},\by_{\overline{G}_i}]} \bigg| \leq \alpha  %\vspace{-0.5cm}
\end{align*}
for a prior distribution $\calP$, where \scalebox{0.9}{$\bz=\calA(\by)$} and \scalebox{0.9}{$\by_{\overline{G}_i}$} is the noisy sequence for data owners outside \scalebox{0.9}{$G_i$}.
}
\begin{proof}
We prove the above by bounding the following equivalent expression for any $i \in [n]$ and $a, b \in \calX$. 
\begin{align*}
     &\frac{\Pr_\calP[\bz | x_i=a, \{y_{G_i}\}, \by_{\overline{G}_i}]}{\Pr_\calP [\bz | x_i=b, \{y_{G_i}\}, \by_{\overline{G}_i}]}\\
     &= \frac{\int \Pr_\calP [\by | x_i = a, \{y_{G_i}\}, \by_{\overline{G}_i}] \Pr_\calA[\bz | \by] d\by}{\int \Pr_\calP [\by | x_i = b, \{y_{G_i}\}, \by_{\overline{G}_i}] \Pr_\calA[\bz | \by] d\by} \\
     &= \frac{\sum_{\sigma \in \mathrm{S}_r} \Pr_\calP[ \sigma(\by_{G_i}^*) | x_i = a, \by_{\overline{G}_i}] \Pr_\calA [\bz | \sigma(\by_{G_i}^*), \by_{\overline{G}_i}] }
     {\sum_{\sigma \in \mathrm{S}_r} \Pr_\calP[ \sigma(\by_{G_i}^*) | x_i = b, \by_{\overline{G}_i}] \Pr_\calA [\bz | \sigma(\by_{G_i}^*), \by_{\overline{G}_i}]} \\
     &\leq \max_{\{ \sigma, \sigma' \in \mathrm{S}_r\}}
     \frac{ \Pr_\calA [\bz | \sigma(\by_{G_i}^*), \by_{\overline{G}_i}]}{ \Pr_\calA [\bz | \sigma'(\by_{G_i}^*), \by_{\overline{G}_i}]}  \\
     &\leq \max_{\{ \sigma, \sigma' \in \mathrm{N}_{G_i}\}}
     \frac{\Pr_\calA[\bz | \sigma(\by)]}{\Pr_\calA[\bz | \sigma'(\by)]}  \\
     &\leq e^\alpha 
\end{align*}
The second line simply marginalizes out the full noisy sequence $\by$. The third line reduces this to a sum over permutations of of $\by_{G_i}$, where $r = |G_i|$ and $\by^*_{G_i}$ is any fixed permutation of values $\{y_{G_i}\}$. This is possible since we are given the values outside the group, $\by_{\overline{G}_i}$, and the unordered set of values inside the group, $\{y_{G_i}\}$. Note that the permutations $\sigma$ written here are possible permutations of the \ldp input, not permutations output by the mechanism applied to the input as sometimes written in other parts of this document. 

The fourth line uses the fact that the numerator and denominator are both convex combinations of $\Pr_\calA [\bz | \sigma(\by_{G_i}^*), \by_{\overline{G}_i}]$ over all $\sigma \in \mathrm{S}_r$. 

The fifth line uses the fact that for any $\by_{\overline{G}_i}$, $$(\sigma(\by_{G_i}^*), \by_{\overline{G}_i}) \approx_{G_i} (\sigma'(\by_{G_i}^*), \by_{\overline{G}_i}) \ . $$
This allows a further upper bound over all neighboring sequences w.r.t. $G_i$, and thus over any permutation of $\by_{\overline{G}_i}$, as long as it is the same in the numerator and denominator. 
\end{proof}

\paragraph{Discussion}
The above Bayesian analysis measures what can be learned about $\DO_i$'s $x_i$ from observing the private release $\bz$ relative to some other known information (the conditioned information). 
%With \ldp alone, we condition on every other data owner's private value $x_j$. This implies that releasing the private sequence $\by$ cannot provide much more information about $x_i$ than releasing every other $\DO_j$'s $x_j$ would. So, only modest information unique to $x_i$ can be garnered by any Bayesian adversary. For Alice, this may be a concern, since making inferences on her disease state from those of her household is indeed a privacy violation. 
Under \name-privacy, we condition on the bag of \ldp values in Alice's group $\{y_{G_i}\}$ as well as the sequence (order and value) of $\ldp$ values outside her group $\by_{\overline{G_i}}$. This implies that releasing the shuffled sequence $\bz$ cannot provide much more information about Alice's $x_i$ than would releasing the \ldp values outside her neighborhood (her group) and the unordered bag of \ldp values inside her neighborhood, regardless of the adversary's prior knowledge $\calP$. This is a communicable guarantee: if Alice feels comfortable with the data collection knowing that her entire neighborhood's responses will be uniformly shuffled together (including those of her household), then she ought to be comfortable with \name-privacy. Now, we have to provide this guarantee to Bob, a neighbor of Alice, as well as Luis, a neighbor of Bob but \textit{not} of Alice. Thus, Bob, Alice and Luis have \textit{distinct} and \textit{overlapping} groups (neighborhoods). Hence, the trivial solution of uniformly shuffling the noisy responses of every group separately does not work in this case. % Of course, we could instead simply shuffle Alice's neighborhood uniformly, but we could not do this for each data owner's group while still maintaining utility -- the analyst can still estimate disease prevalence within neighborhoods and districts. With overlapping groups this may require shuffling the entire dataset uniformly. 
\name-privacy, however, offers the above guarantee to each user (knowing that their entire neighborhood is \textit{nearly} uniformly shuffled) while still maintaining utility (estimate disease prevalence within neighborhoods). Semantically, this is very powerful, since it implies that the noisy responses specific to one's household cannot be leveraged to infer one's disease state $x_i$. %The adversary can learn about $x_i$ from the disease prevalence outside Alice's neighborhood and, on average, inside her neighborhood, but not much beyond that. 

\subsection{Proof of Theorem \ref{thm: decision theoretic}}

\textbf{Theorem \ref{thm: decision theoretic}} 

\emph{
  For $\mathcal{A}(\mathcal{M}(\bx))=\bz$ where $\mathcal{M}(\cdot)$ is $\epsilon$-\ldp and $\mathcal{A}(\cdot)$ is $\alpha$ - \name private, we have  
 \begin{align*}
     \Pr[\mathcal{D}_{Adv} \text{ loses}] \geq \lfloor \frac{r-k}{k} \rfloor e^{-(2k\epsilon+\alpha)} \cdot \Pr[\mathcal{D}_{Adv} \text{ wins}]
 \end{align*}
 for any input subgroup $I \subset G_i, r = |G_i|$ and  $k < r/2$. 
 }
 
 \begin{proof}
 
  We first focus on deterministic adversaries and then expand to randomized adversaries afterwards using the fact that randomized adversaries are mixtures of deterministic ones. 

Our adversary $\mathcal{D}_{Adv}$ is then defined by a deterministic decision function $\eta: \calY^n \rightarrow [n]^k$. Upon observing $\bz$, $\eta(\bz)$ selects $k$ elements in $\bz$ which it believes originated from $I \subset G_i$. 
 
 In the following, let $\Pr_{\bz}$ be the probability of events conditioned on the shuffled output sequence $\bz$, where randomness is over the $\epsilon$-\ldp mechanism $\mathcal{M}$ and the $\alpha$-\name-private shuffling mechanism $\calA$. \footnote{As an abuse of notation, we assume the output space of the \ldp randomizers, $\calY$, have outcomes with non-zero measure e.g. randomized response. The following analysis can be expanded to continuous outputs (with outcomes of zero measure) by simply replacing the output sequence $\bz \in \calY^n$ with an output event $\mathbf{Z} \subseteq \calY^n$.}
 
 The adversary wins if it reidentifies $> \frac{k}{2}$ of the \ldp values originating from $I$. Let $H = \eta(\bz)$ be the indices of elements in $\bz$ selected by $\eta$. Let $W = \{\sigma \in \mathrm{S}^n : |\sigma(H) \cap I| > \frac{k}{2}\}$ be the set of permutations where the adversary wins and let $L = \{\sigma \in \mathrm{S}^n :\sigma(H) \cap I| \leq \frac{k}{2}\} $ be the set of permutations where the adversary loses. 
 \begin{align*}
     \Pr_{\bz} [\eta(\bz) \text{ wins}] &= \Pr_{\bz} [\sigma \in W] \\
     \Pr_{\bz} [\eta(\bz) \text{ loses}] &= \Pr_{\bz} [\sigma \in L] 
 \end{align*}
 where $\sigma$ is the shuffling permutation produced by $\calA$, $\bz = \sigma(\by)$ i.e. $z_i = y_{\sigma(i)}$. Concretely, this is equivalent to $\DO_i$ releasing $\DO_{\sigma(i)}$'s \ldp response. Since the permutation and \ldp outputs are randomized, many subgroups of size $k$ in $G_i$ could have produced the \ldp values $(z_{H_1}, \dots, z_{H_k})$ and then been mapped to $H$ by a permutation. Concretely, there is a reasonable probability that Alice's household output the \ldp values of another $k$-member household in her neighborhood and they output her household's \ldp values. In the worst case, this is $e^{-2k \epsilon}$ less likely than without swapping values, by group \DP guarantees. Since both households are part of the same group $G_i$, the permutation that maps her household to elements $H$ in the output is close in probability to that which maps the other household to elements $H$ in the output. As such, we have in the worst case a $e^{-(2k\epsilon + \alpha)}$ reduction in probability of the other household having swapped \ldp values with Alice's and permuting to subset $H$. 
 
 The above provides intuition on how we could get the same output $\bz$ many different ways, and how Alice's household could or could not contribute to elements $H$. It does not, however, explain why an adversary who is given output $\bz$ has limited advantage in choosing a subset $H$ such that they recover \emph{most} of Alice's household's values. We formalize this fact as follows. 
 
 We may rewrite the probabilities of winning or losing by marginalizing out all possible $\ldp$ sequences $\by$. Conditioning on the output sequence $\bz$, the only possible \ldp sequences $\by$ are permutations of $\bz$. Note that the probability of any sequence $\by$ is determined by the input $\bx$ and the \ldp mechanism $\mathcal{M}$:  
 \begin{align*}
    \Pr_\bz[\eta(\bz) \text{ loses}]
     &= \Pr_{\bz} [\sigma \in W] \\
     &= \sum_{\sigma \in W}
     \Pr [\calA(\bx) = \by = \sigma^{-1}(\bz)]  \Pr [ \sigma | \by ] / \Pr[\bz] 
 \end{align*}
 Note that $\Pr_\bz [\sigma | \by] = \Pr_\bz [\sigma]$ for the mallows mechanism, which chooses its permutations independently of $\by$. Now consider when $\eta(\bz)$ loses. By similar arguments as above: 
\begin{align*}
    \Pr_\bz[\eta(\bz) \text{ loses}]
    &= \Pr_\bz[\sigma \in L] \\
    &= \sum_{\sigma \in L}
     \Pr [\calA(\bx) = \by = \sigma^{-1}(\bz)]  \Pr [ \sigma | \by ] / \Pr[\bz]
\end{align*}
The odds of losing versus winning is given by 
 \begin{align*}
    \frac{\Pr_\bz[\eta(\bz) \text{ loses}]}{\Pr_\bz[\eta(\bz) \text{ wins}]}
    &= \frac{ \sum_{\sigma' \in L} \Pr [\calA(\bx) = \by = \sigma^{'-1}(\bz)]  \Pr [ \sigma' | \by ] }
    { \sum_{\sigma \in W} \Pr [\calA(\bx) = \by = \sigma^{-1}(\bz)]  \Pr [ \sigma | \by ] } \\
\end{align*}
We now show that for each $\sigma$ in the denominator, we may construct $m = \lfloor \frac{r-k}{k} \rfloor$ distinct permutations $\sigma'$ in the numerator that are close in probability to it. 
\begin{lemma}
For every $\sigma \in W$ there exists a set of $m = \lfloor \frac{r-k}{k} \rfloor$ permutations, $E(\sigma)$, such that 
\begin{enumerate}
    \item $E(\sigma) \subseteq L$
    \item $\sigma^{-1} \approx_{G_i} \sigma^{'-1}$ 
    \item $E(\sigma_a) \cap E(\sigma_b) = \emptyset$ for any pair $\sigma_a, \sigma_b \in W$
    \item $\Pr [\calA(\bx) = \by = \sigma^{-1}(\bz)] \leq e^{2k\epsilon} \Pr [\calA(\bx) = \by = \sigma^{'-1}(\bz)]$ for any $\bx \in \calX^n$ and any $\bz \in \calY^n$
\end{enumerate}
\end{lemma}
\begin{proof}
Given $\sigma \in W$, we construct $E(\sigma)$ by first taking the inverse $\sigma^{-1}$. Recall that, since $\sigma \in W$, we have that $|\sigma^{-1}(I) \cap H| > \frac{k}{2}$. ($\sigma^{-1}(i) = j$ could be interpreted as data owner $i$'s \ldp value will be output at position $j$). We then divide the remainder of the group $G_i \backslash I$ into $m$ disjoint subsets of size $k$ each, $J_1, J_2, \dots, J_m$. These represent the other distinct subsets of size $k$ that Alice's household could swap \ldp values with. We then produce $m$ permutations, $\sigma^{'-1}_1, \dots, \sigma^{'-1}_m$, by making $\sigma^{'-1}_i(I) = \sigma^{-1}(J_i)$ and $\sigma^{'-1}_i(J_i) = \sigma^{-1}(I)$ (preserving order within those subsets) and $\sigma^{'-1} = \sigma^{-1}$ everywhere else. 

On the first point, we know that every $\sigma' \in E(\sigma)$ is also in $L$. We know this because $\sigma^{'-1}_i(I) = \sigma^{-1}(J_i)$. Since $\sigma \in W$, we have that $|\sigma^{-1}(J_i) \cap H| < \frac{k}{2}$ since $|\sigma^{-1}(I) \cap H| \geq \frac{k}{2}$ and $I \cap J_i = \emptyset$ by definition. Thus, $|\sigma^{'-1}_i(I) \cap H| < \frac{k}{2}$, so $|\sigma'_i(H) \cap I| < \frac{k}{2}$ and $\sigma'_i \in L$. 

On the second point, we know that the inverse permutations are neighboring $\sigma^{-1} \approx_{G_i} \sigma^{'-1}$ simply by construction -- they only differ on elements in $G_i$. 

On the third point, we know that the sets $E(\sigma_a)$ and $E(\sigma_b)$ are distinct since we can map any permutation $\sigma' \in E(\sigma_a)$ uniquely back to $\sigma_a$ for any $\sigma_a \in W$. We do so by taking its inverse $\sigma^{'-1}$, finding which subset $J_i$ has majority elements from $H$ i.e. $|\sigma^{'-1}(J_i) \cap H| > \frac{k}{2}$. Swap elements back: $\sigma^{'-1}(J_i)$ with $\sigma^{'-1}(I)$. Invert back to $\sigma_a$. 

On the fourth point, we know that $\sigma^{-1}(\bz)$ and $\sigma^{'-1}(\bz)$ differ on at most $2k$ indices. As such, by group \DP guarantees, we know that their probabilities must be close to a factor of $e^{-2k\epsilon}$ regardless of $\bz$ and $\bx$. 
\end{proof}

Using the above Lemma we may bound the odds of losing vs. winning. 

\begin{align*}
    \frac{\Pr_\bz[\eta(\bz) \text{ loses}]}{\Pr_\bz[\eta(\bz) \text{ wins}]}
    &= \frac{ \sum_{\sigma' \in L} \Pr [\calA(\bx) = \by = \sigma^{'-1}(\bz)]  \Pr [ \sigma' | \by ] }
    { \sum_{\sigma \in W} \Pr [\calA(\bx) = \by = \sigma^{-1}(\bz)]  \Pr [ \sigma | \by ] } \\
    &\geq \frac{\sum_{\sigma \in W} \sum_{\sigma' \in E(\sigma)} \Pr [\calA(\bx) = \by = \sigma^{'-1}(\bz)]  \Pr [ \sigma' | \by ]}
    {\sum_{\sigma \in W} \Pr [\calA(\bx) = \by = \sigma^{-1}(\bz)]  \Pr [ \sigma | \by ]} \\
    &\geq \min_{\sigma \in W} \frac{\sum_{\sigma' \in E(\sigma)} \Pr [\calA(\bx) = \by = \sigma^{'-1}(\bz)]  \Pr [ \sigma' | \by ]}{\Pr [\calA(\bx) = \by = \sigma^{-1}(\bz)]  \Pr [ \sigma | \by ]} \\
    &\geq \lfloor \frac{r-k}{k} \rfloor e^{-(2k\epsilon + \alpha)}
\end{align*}
where the last line follows from the fourth point of the above Lemma (for the $2k\epsilon$ term) and the fact that the inverse permutations $\sigma'^{-1}, \sigma^{-1}$ are neighboring (second point of the Lemma) so the probabilities of the mechanism to produce $\sigma$ vs. $\sigma'$ to reach $\bz$ from these neighboring permutations must be close by a factor of $e^{\alpha}$. 

Since the above holds for any $\bz$ and $\bx$, the bound holds on average across all outcomes $\bz$, thus 
\begin{align*}
    \Pr[\eta \text{ loses}] \geq \lfloor \frac{r-k}{k} \rfloor e^{-(2k\epsilon+\alpha)} \cdot \Pr[\eta \text{ wins}]
\end{align*}
for any deterministic adversary with decision function $\eta$. Finally, we may write any probabilistic adversary as mixture of decision functions. By convexity (same argument used in Lemma \ref{lem:convexity}), the above bound still holds. As such, 

\begin{align*}
    \Pr[\mathcal{D}_{Adv} \text{ loses}] \geq \lfloor \frac{r-k}{k} \rfloor e^{-(2k\epsilon+\alpha)} \cdot \Pr[\mathcal{D}_{Adv} \text{ wins}]
\end{align*}

 \end{proof}
 
 \subsection{Utility of Shuffling Mechanism}\label{app:utility}
 We now introduce a novel metric, $(\eta,\delta)$-preservation, for assessing the utility of any shuffling mechanism. Let $S\subseteq [n]$ correspond to a set of indices in $\by$. The metric is defined as follows.
%representing data owners in Alice's neighborhood for instance.  
%$(\eta,\delta)$-preservation measures how well the shuffling mechanism preserves the original indices in $S$ after shuffling, i.e. the fraction of data owners in Alice's neighborhood that still correspond to datapoints from the neighborhood after shuffling:
\begin{defn}($(\eta,\delta)$-preservation) A shuffling mechanism $\calA:\calY^n\mapsto\calY^n$ is defined to be $(\eta,\delta)$-preserving $(\eta, \delta 
\in [0,1])$ w.r.t to a given subset $S\subseteq [n]$, if \begin{gather}\Pr\big[|S_{\sigma}\cap S|\geq \eta\cdot|S|\big]\geq 1-\delta,  \sigma \in \mathrm{S}_n\end{gather} where $\bz=\calA(\by)=\sigma(\by)$ and $S_{\sigma}=\{\sigma(i)|i \in S\}$. \label{def:utility} 
% \vspace{-0.2cm}
\end{defn}
For example, consider \scalebox{0.9}{$S=\{1,4,5,7,8\}$}. If \scalebox{0.9}{$\calA(\cdot)$} permutes the output according to  \scalebox{0.9}{$\sigma=(\underline{5}\:3\:2\:\underline{6}\:\underline{7}\:9\:\underline{8}\:\underline{1}\:4\:10)$}, then  \scalebox{0.9}{$S_{\sigma}=\{5,6,7,8,1\}$}  which preserves \scalebox{0.9}{$4$} or \scalebox{0.9}{$80\%$} of its original indices.  This means that for any data sequence $\by$, at least \scalebox{0.9}{$\eta$} fraction of its data values corresponding to the subset \scalebox{0.9}{$S$} overlaps with that of shuffled sequence $\bz$ with high probability \scalebox{0.9}{$(1-\delta)$}. Assuming, \scalebox{0.9}{$\{y_S\}=\{y_{i}|i
\in S\}$} and \scalebox{0.9}{$\{z_S\}=\{z_i|i \in S\}=\{y_{\sigma(i)}| i \in S\}$} denotes the set of data values corresponding to $S$ in data sequences $\by$ and $\bz$ respectively, we  have \scalebox{0.9}{$\Pr\big[|\{y_S\}\cap \{z_S\}|\geq \eta \cdot |S|\big]\geq 1-\delta, \: \forall \by $}.
For example, let $S$ be the set of individuals from Nevada. Then, for a shuffling mechanism that provides \scalebox{0.9}{$(\eta =0.8, \delta=0.1)$}-preservation to $S$, with probability \scalebox{0.9}{$\geq 0.9$}, \scalebox{0.9}{$\geq 80\%$} of the values that are reported to be from Nevada in $\bz$ are genuinely from Nevada. The rationale behind this metric is that it captures the utility of the learning allowed by \name-privacy -- if $S$ is equal to some group \scalebox{0.9}{$G \in \calG$}, \scalebox{0.9}{$(\eta, \delta)$} preservation allows overall statistics of \scalebox{0.9}{$G$} to be captured. Note that this utility metric is \textit{agnostic of both the data distribution and the analyst's query}. Hence, it is a conservative analysis of utility which serves as a lower bound for learning from $\{z_S\}$. We suspect that with the knowledge of the data distribution and/or the query, a tighter utility analysis is possible. \\
A formal utility analysis of Alg. \ref{algo:main} is presented in App. \ref{app:utility:formal}. Empirical evaluation of $(\eta,\delta)$ - preservation is presented in App. \ref{app:extraresults}. 
\subsection{Discussion on Properties of Mallows Mechanism}\label{app:prop}

% \textbf{Property \ref{prop:1}} 
\begin{prope}
\label{prop:1}
% \emph{
For group assignment $\calG$, a  mechanism $\calA(\cdot)$ that shuffles according to a permutation sampled from the Mallows model $\mathbb{P}_{\theta,\textswab{d}}(\cdot)$, satisfies $(\alpha, \calG)$-\name privacy where
\begin{align*}
 \Delta(\sigma_0 : \textswab{d}, \calG) &= \max_{(\sigma, \sigma') \in N_\calG} |\textswab{d}(\sigma_0 \sigma, \sigma_0) - \textswab{d}(\sigma_0 \sigma', \sigma_0)|\\
 \text{and}&\\
    \alpha 
    &= \theta \cdot \Delta(\sigma_0 : \textswab{d}, \calG)
\end{align*} 
We refer to $\Delta(\sigma_0 : \textswab{d}, \calG) $ as the sensitivity of the rank-distance measure $\textswab{d}(\cdot)$
\end{prope}
% }
\begin{proof}
% Start with a useful lemma. 
% \begin{lemma}
% Consider any pair of neighboring permutations $\sigma \approx_{G_i} \sigma'$ w.r.t. group $G_i$. Consider any third permutation $\sigma^* \in \mathrm{S}_n$. Then, the permutations that turn $\sigma$ and $\sigma'$ into $\sigma^*$ are also neighboring w.r.t. $G_i$. Formally, if 
% \begin{align*}
%     \sigma^*(\by) &= \sigma_a \big( \sigma(\by) \big) \\
%     \sigma^*(\by) &= \sigma_b \big( \sigma'(\by) \big)  
% \end{align*}
% then, $\sigma_a \approx_{G_i} \sigma_b$.

Consider two permutations of the initial sequence $\by$, $\sigma_1(\by), \sigma_2(\by)$ that are neighboring w.r.t. some group $G_i \in \calG$, $\sigma_1 \approx_{G_i} \sigma_2$. Additionally consider any fixed released shuffled sequence $\bz$. Let $\Sigma_1, \Sigma_2$ be the set of permutations that turn $\sigma_1(\by), \sigma_2(\by)$ into $\bz$, respectively: 
\begin{align*}
    \Sigma_1 
    & = \{\sigma \in \mathrm{S}_n : \sigma \sigma_1(\by) = \bz \} \\
    \Sigma_2 
    & = \{\sigma \in \mathrm{S}_n : \sigma \sigma_2(\by) = \bz \} \quad .
\end{align*}
In the case that $\{y\}$ consists entirely of unique values, $\Sigma_1, \Sigma_2$ will each contain exactly one permutation, since only one permutation can map $\sigma_i(\by)$ to $\bz$. 

\begin{lemma}
For each permutation $\sigma_1' \in \Sigma_1$ there exists a permutation in $\sigma_2' \in \Sigma_2$ such that 
\begin{align*}
    \sigma_1' \approx_{G_i} \sigma_2' \quad . 
\end{align*}
\end{lemma}
Proof follows from the fact that --- since only the elements $j \in G_i$ differ in $\sigma_1(\by)$ and $\sigma_2(\by)$ --- only those elements need to differ to achieve the same output permutation. In other words, we may define $\sigma_1', \sigma_2'$ at all inputs $i \notin G_i$ identically, and then define all inputs $i \in G_i$ differently as needed. As such, they are neighboring w.r.t. $G_i$. 

Recalling that Alg. 1 applies $\sigma_0^{-1}$ to the sampled permutation, we must sample $\sigma_0\sigma_1'$ (for some $\sigma_1' \in \Sigma_1$) for the mechanism to produce $\bz$ from $\sigma_1(\by)$. Formally, since $\sigma_1' \sigma_1 (\by) = \bz$ we must sample $\sigma_0 \sigma_1'$ to get $\bz$ since we are going to apply $\sigma_0^{-1}$ to the sampled permutation. 
\begin{align*}
    \Pr\big[ \calA \big( \sigma_1(\by) \big) = \bz \big] 
    &= \mathbb{P}_{\theta,\textswab{d}}\big(\sigma_0\sigma', \sigma' \in \Sigma_1 : \sigma_0\big) \\
    \Pr\big[ \calA \big( \sigma_2(\by) \big) = \bz \big] 
    &= \mathbb{P}_{\theta,\textswab{d}}\big(\sigma_0\sigma', \sigma' \in \Sigma_2 : \sigma_0\big) 
\end{align*}

Taking the odds, we have
\begin{align*}
     \frac{\mathbb{P}_{\theta,\textswab{d}}\big(\sigma_0\sigma', \sigma' \in \Sigma_1 : \sigma_0\big)}{
     \mathbb{P}_{\theta,\textswab{d}}\big(\sigma_0\sigma'', \sigma'' \in \Sigma_2 : \sigma_0\big)} 
    &= \frac{\sum_{\sigma' \in \Sigma_1}\mathbb{P}_{\Theta,\textswab{d}}(\sigma_0\sigma' : \sigma_0)}{
    \sum_{\sigma'' \in \Sigma_2}\mathbb{P}_{\Theta,\textswab{d}}(\sigma_0\sigma'' : \sigma_0)} \\
    &=\frac{\sum_{\sigma' \in \Sigma_1} e^{-\theta \textswab{d}(\sigma_0\sigma', \sigma_0)}}{
    \sum_{\sigma'' \in \Sigma_2} e^{-\theta \textswab{d}(\sigma_0\sigma'', \sigma_0)}}\\
    &\leq \frac{e^{-\theta \textswab{d}(\sigma_0\sigma_a, \sigma_0)}}{e^{-\theta \textswab{d}(\sigma_0\sigma_b, \sigma_0)}} \\
    &\leq e^{\theta|  \textswab{d}(\sigma_0\sigma_a, \sigma_0) - \textswab{d}(\sigma_0\sigma_b, \sigma_0) |} \\
    &\leq e^{\theta \Delta}
\end{align*}
where 
\begin{align*}
    \sigma_a &= \arg \max_{\sigma' \in \Sigma_1} e^{-\theta \textswab{d}(\sigma_0 \sigma', \sigma_0)} 
    \text{ and } \\
    \sigma_a &= \arg \min_{\sigma'' \in \Sigma_2} e^{-\theta \textswab{d}(\sigma_0 \sigma'', \sigma_0)} ~.
\end{align*}
Therefore, setting $\alpha = \theta \cdot \Delta$, we achieve $(\alpha, \calG)$-\name privacy. 
\end{proof}

% \textbf{Property \ref{prop:2}}
% \emph{
\begin{prope}
\label{prope:2}
The sensitivity of a rank-distance is an increasing function of the width $\omega_{\calG}^{\sigma_0}$. For instance, for Kendall's $\tau$ distance $\textswab{d}_\tau(\cdot )$, we have 
$\Delta(\sigma_0 : \textswab{d}_\tau, \calG)
    =\frac{\omega_{\calG}^{\sigma_0}(\omega_{\calG}^{\sigma_0} + 1)}{2}$. 
% }
\end{prope}
%  \begin{lemma}
%  For any two neighboring permutations $\sigma \approx_{\calG} \sigma'$: 
%  \begin{align*}
%      \bigg| \log \frac{\mathbb{P}_{\Theta,\textswab{d}}(\sigma:\sigma_0)}{\mathbb{P}_{\Theta,\textswab{d}}(\sigma':\sigma_0)} \bigg| 
%      &\leq \theta \Delta(\sigma_0 : \textswab{d}, \calG) 
%  \end{align*}\label{lem:prop1}
%  \end{lemma}
 
%  The above lemma essentially gives the proof for Prop. \ref{prop:1}.
 
%  \begin{lemma}For Kendall's $\tau$ distance,
%  \begin{align*}
%      \bigg| \log \frac{\mathbb{P}_{\Theta,\textswab{d}}(\sigma_s:\sigma_0)}{\mathbb{P}_{\Theta,\textswab{d}}(\sigma_s':\sigma_0)} \bigg| \leq \theta \Delta(\sigma_0 : \textswab{d}, \calG)
%  \end{align*}\end{lemma}

To show the sensitivity of Kendall's $\tau$, we make use of its triangle inequality. 

 \begin{proof}
 Recall from the proof of the previous property that the expression $\textswab{d}(\sigma, \sigma_0) = \textswab{d}\big( \sigma_0\sigma, \sigma_0 \big)$, where $\textswab{d}$ is the actual rank distance measure e.g. Kendall's $\tau$. As such, we require that 
 
 \begin{align*}
     \big|  \textswab{d}(\sigma_0\sigma_a, \sigma_0) - \textswab{d}(\sigma_0\sigma_b, \sigma_0) \big|
     &\leq \frac{\omega_{\calG}^{\sigma_0}(\omega_{\calG}^{\sigma_0} + 1)}{2}
 \end{align*}

for any pair of permutations $(\sigma_a, \sigma_b) \in N_\calG$. 
 
 For any group $G_i \in \calG$, let $W_i \subseteq n$ represent the smallest contiguous subsequence of indices in $\sigma_0$ that contains all of $G_i$. 

For instance, if $\sigma_0 = [2,4,6,8,1,3,5,7]$ and $G_i = \{2,6,8\}$, then $W_i = \{2,4,6,8\}$. Then the group width width is $\omega_i = |W_i| - 1 = 3$. Now consider two permutations neighboring w.r.t. $G_i$, $\sigma_a \approx_{G_i} \sigma_b$, so only the elements of $G_i$ are shuffled between them. We want to bound 
\begin{align*}
    \big|  \textswab{d}(\sigma_0\sigma_a, \sigma_0) - \textswab{d}(\sigma_0\sigma_b, \sigma_0) \big| 
\end{align*}
For this, we use a pair of triangle inequalities: 
\begin{align*}
    \textswab{d}(\sigma_0\sigma_a, \sigma_0\sigma_b) 
    &\geq \textswab{d}(\sigma_0\sigma_a, \sigma_0) - \textswab{d}(\sigma_0\sigma_b, \sigma_0) 
    \quad \& \quad 
    \textswab{d}(\sigma_0\sigma_a, \sigma_0\sigma_b) 
    &\geq \textswab{d}(\sigma_0\sigma_b, \sigma_0) - \textswab{d}(\sigma_0\sigma_a, \sigma_0)
\end{align*}
so, 
\begin{align*}
     \big|  \textswab{d}(\sigma_0\sigma_a, \sigma_0) - \textswab{d}(\sigma_0\sigma_b, \sigma_0) \big| &\leq \textswab{d}(\sigma_0\sigma_a, \sigma_0\sigma_b)
\end{align*}

% We make use of the fact that by only permuting the contents of $W_i$ in $\sigma_0$, we can construct some $\sigma_0'$ such that $\textswab{d}(\sigma , \sigma_0) = \textswab{d}(\sigma' , \sigma_0')$. By the triangle inequality, we have
% \begin{align*}
% \textswab{d}(\sigma' , \sigma_0) &\leq \textswab{d}(\sigma' , \sigma_0') + \textswab{d}(\sigma_0 , \sigma_0') \\
% &= \textswab{d}(\sigma , \sigma_0) + \textswab{d}(\sigma_0 , \sigma_0') \\
% \textswab{d}(\sigma' , \sigma_0) - \textswab{d}(\sigma , \sigma_0) 
% &\leq \textswab{d}(\sigma_0 , \sigma_0')
% \end{align*}
% Thus,
% \begin{align*}
%     |\textswab{d}(\sigma , \sigma_0) - \textswab{d}(\sigma' , \sigma_0) |
%     &\leq |\textswab{d}(\sigma_0 , \sigma_0')|
% \end{align*}
Since $\sigma_0\sigma_a$ and $\sigma_0\sigma_b$ only differ in the contiguous subset $W_i$, the largest number of discordant pairs between them is given by the maximum Kendall's $\tau$ distance between two permutations of size $\omega_i + 1$:  
\begin{align*}
    |\textswab{d}(\sigma_0\sigma_a , \sigma_0\sigma_b)|
    &\leq \frac{\omega_i(\omega_i + 1)}{2}
\end{align*}
Since $\omega_{\calG}^{\sigma_0} \geq \omega_i$ for all $G_i \in \calG$, we have that 
\begin{align*}
    \Delta(\sigma_0 : \textswab{d}, \calG) \leq 
    \frac{\omega_{\calG}^{\sigma_0}(\omega_{\calG}^{\sigma_0} + 1)}{2}
\end{align*}
\end{proof}

\subsection{Hardness of Computing The Optimum Reference Permutation}\label{app:NP}
\begin{thm} The problem of finding the optimum reference permutation, i.e., $\sigma_0=\arg\min_{\sigma\in \mathrm{S}_n}\omega_{\calG}^{\sigma}$ is NP-hard. \label{thm:NP} \end{thm}
\begin{proof} We start with the formal
representation of the problem as follows.

\textit{Optimum Reference Permutation Problem.} Given n subsets $\calG=\{G_i\in 2^{[n]}, i \in [n]\}$,  find the permutation $\sigma_0=\arg\min_{\sigma\in \mathrm{S}_n}\omega_{\calG}^{\sigma}$.  

Now, consider the following job-shop scheduling problem.

\textit{Job Shop Scheduling.} There is one job $J$ with $n$ operations $o_i, i \in [n]$ and $n$ machines such that $o_i$ needs to run on machine $M_i$.  Additionally, each machine has a sequence dependent processing time $p_i$. Let $S$ be the sequence till 
There are $n$ subsets $S_i \subseteq [n]$, each corresponding to a set of operations that need to occur in contiguous machines, else the processing times incur penalty as follows. Let $p_i$ denote the processing time for the machine running the $i$-th operation scheduled. Let $\mathbb{S}_i$ be the prefix sequence  with $i$ schedulings. For instance, if the final scheduling is $ 1\:3\:4\:5\:9\:8\:10\:6\:7\:2$ then $\mathbb{S}_4=1345$. Additionally, let $P^j_{\mathbb{S}_i}$ be the shortest subsequence such of $\mathbb{S}_i$ such that it contains all the elements in $S_j \cap \{\mathbb{S}_{i}\}$. For example for $S_1=\{3,5,7\}$, $P^1_{\mathbb{S}_{4}}=345$. 
\begin{gather}p_i=\max_{i\in [n]}(|P^j_{\mathbb{S}_i}|-|S_j \cap \{\mathbb{S}_i\}|)\end{gather}
 The objective is to find a scheduling for $J$ such that it minimizes the makespan, i.e., the completion time of the job. Note that $p_n=\max_i{p_i}$, hence the problem reduces to minimizing $p_n$.

\begin{lemma}The aforementioned  job shop scheduling problem with sequence-dependent processing time is NP-hard.\end{lemma}
\begin{proof} Consider the following instantiation of the sequence-dependent job shop scheduling problem where the processing time is given by $p_i$=$p_{i-1}+w_{kl}, p_1=0$ where $\mathbb{S}_i[i-1]=k$, $\mathbb{S}_i[i]=l$ and $w_{ij}, j\in S_i$  represents some associated weight.    This problem is equivalent to the travelling salesman problem (TSP) \cite{TSP} and is therefore, NP-hard. Thus, our aforementioned job shop scheduling problem is also clearly NP-hard. \end{proof}

\textit{Reduction:}
Let the $n$ subsets $S_i$ correspond to the groups in $\calG$. Clearly, minimizing $\omega^{\sigma}_{\calG}$ minimizes $p_n$. Hence, the optimal reference permutation gives the solution to the scheduling problem as well. 

 \end{proof}

 \begin{figure}[h]
\begin{subfigure}[b]{\columnwidth}\centering
    \includegraphics[width=0.7\linewidth]{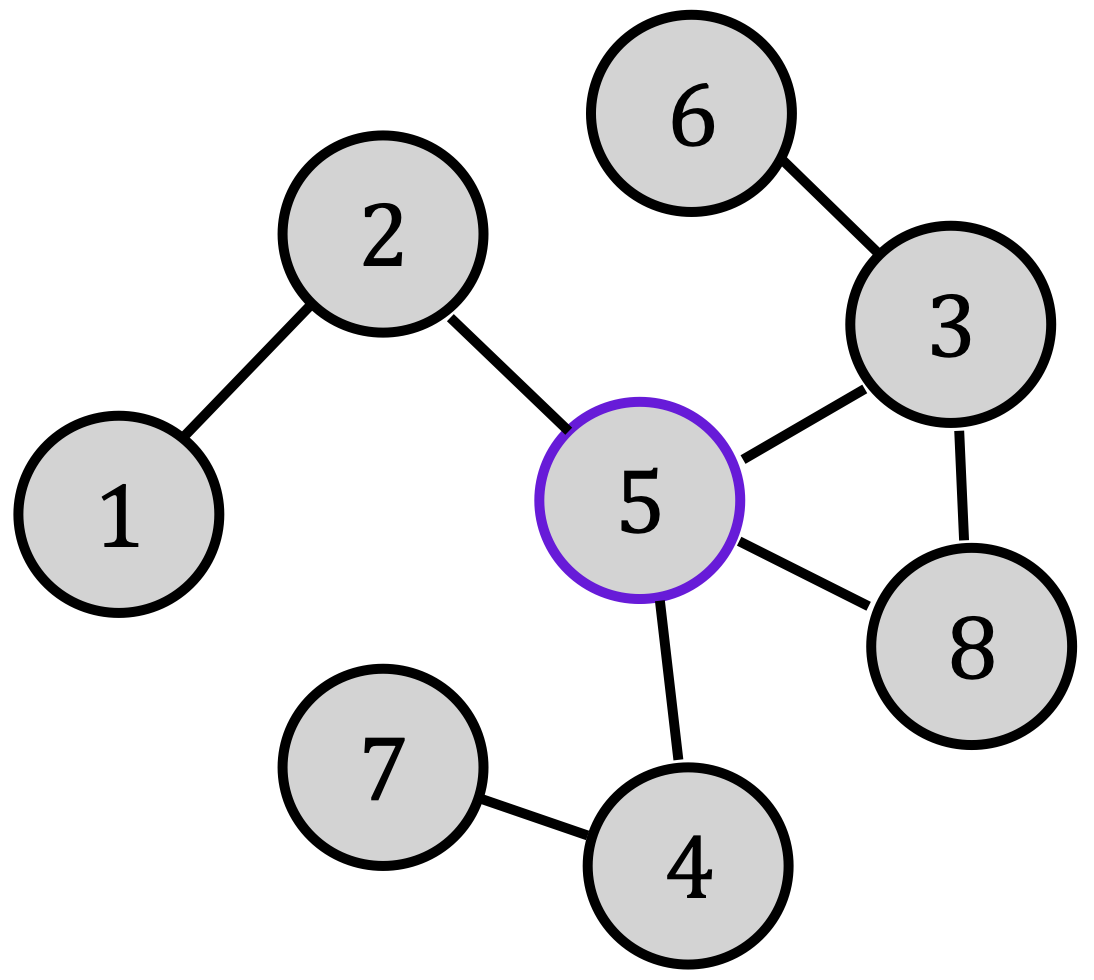}
        \caption{Group graph}
        \label{fig:BFS:graph}
    \end{subfigure}\\
    \begin{subfigure}[b]{\columnwidth}\centering
    \includegraphics[width=0.5\columnwidth]{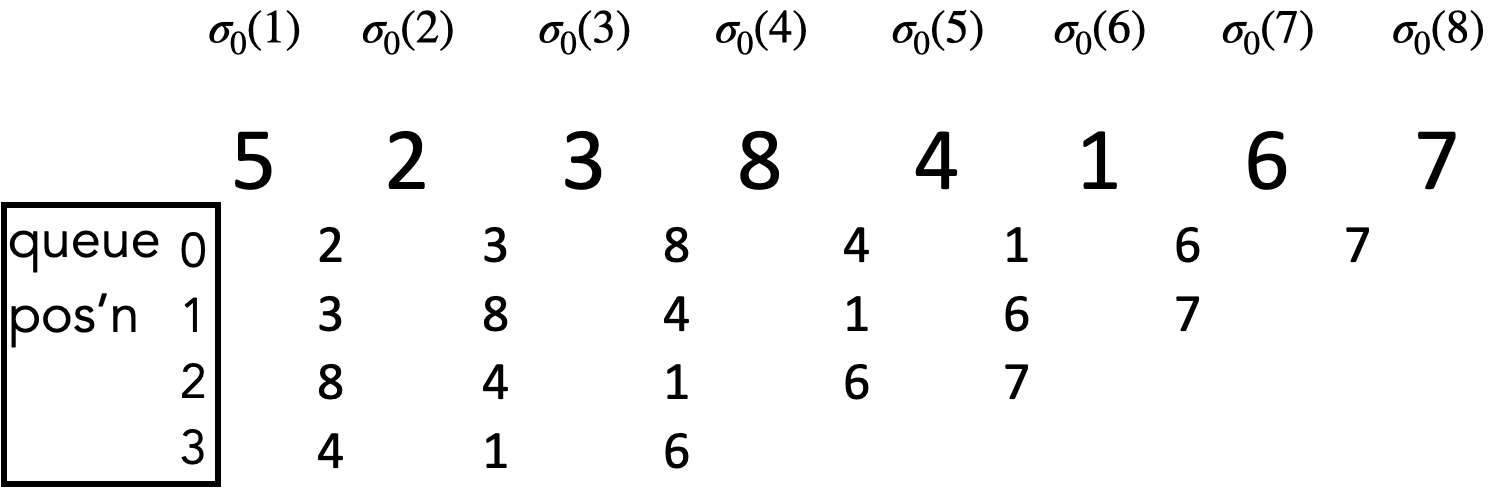}
        \caption{BFS reference permutation $\sigma_0$}
        \label{fig:BFS:order}
    \end{subfigure}
    \caption{Illustration of Alg. 1}
    \label{fig:alg:example}
\end{figure}

 \subsection{Illustration of Alg. 1}\label{app:alg:illustration}
 We now provide a small-scale step-by-step example of how Alg. 1 operates. 
 
 Fig. \ref{fig:BFS:graph} is an example of a grouping $\calG$ on a dataset of $n = 8$ elements. The group of $\DO_i$ includes $i$ and its neighbors. For instance, $G_8 = \{8,3,5\}$. To build a reference permutation, Alg. 1 starts at the index with the largest group, $i = 5$ (highlighted in purple), with $G_5 = \{5,2,3,8,4\}$. As shown in Figure \ref{fig:BFS:order}, the $\sigma_0$ is then constructed by following a BFS traversal from $i=5$. Each $j \in G_5$ is visited, queuing up the neighbors of each $j \in G_5$ that haven't been visited along the way, and so on. The algorithm completes after the entire graph has been visited. 
 
 The goal is to produce a reference permutation in which the width of each group in the reference permutation $\omega_i$ is small. In this case, the width of the largest group $G_5$ is as small as it can be $\omega_5 =5-1 = 4$. However, the width of $G_4 = \{4,5,7\}$ is the maximum possible since $\sigma^{-1}(5) = 1$ and $\sigma^{-1}(7) = 8$, so $\omega_4 = 7$. This is difficult to avoid when the maximum group size is large as compared to the full dataset size $n$. Realistically, we expect $n$ to be significantly larger, leading to relatively smaller groups. 
 
 With the reference permutation in place, we compute the sensitivity: 
 \begin{align*}
     \Delta(\sigma_0 : \textswab{d}, \calG)
     &= \frac{\omega_4 (\omega_4 + 1)}{2} \\
     &= 28
 \end{align*}
 Which lets us set $\theta = \frac{\alpha}{28}$ for any given $\alpha$ privacy value. To reiterate, lower $\theta$ results in more randomness in the mechanism. 
 
 We then sample the permutation $\hat{\sigma} = \mathbb{P}_{\theta, \textswab{d}}(\sigma_0)$. Suppose 
 \begin{align*}
     \hat{\sigma}
     &= [3 \ 2\ 5\ 4\ 8\ 1\ 7\ 6]
 \end{align*}
 Then, the released $\bz$ is given as 
  \begin{align*}
     \bz = \sigma^* &= \sigma^{-1} \hat{\sigma} (\by)\\
     &= [y_1 \ y_2\ y_5\ y_8\ y_3\ y_7\ y_6\ y_4]
 \end{align*}
 One can think of the above operation as follows. What was 5 in the reference permutation ($\sigma_0(1) = 5$) is 3 in the sampled permutation $(\hat{\sigma}(1) = 3)$. So, index 5 corresponding to $\DO_5$ now holds $\DO_3$'s noisy data $y_3$. As such, we shuffle mostly between members of the same group, and minimally between groups. 

\newcommand{\calO}{\mathcal{O}}

 The runtime of this mechanism is dominated by the Repeated Insertion Model sampler \cite{RIM}, which takes $\calO(n^2)$ time. It is very possible that there are more efficient samplers available, but RIM is a standard and simple to implement for this first proposed mechanism. Additionally, the majority of this is spent computing sampling parameters which can be stored in advanced with $\calO(n^2)$ memory. Furthermore, sampling from a Mallows model with some reference permutation $\sigma_0$ is equivalent to sampling from a Mallows model with the identity permutation and applying it to $\sigma_0$. As such, permutations may be sampled in advanced, and the runtime is dominated by computation of $\sigma_0$ which takes $\calO(|V| + |E|)$ time (the number of veritces and edges in the graph). 
 
 \subsection{Proof of Thm. \ref{thm:privacy}}\label{app:thm:privacy}
 \textbf{Theorem \ref{thm:privacy}}
 \emph{Alg. 1 is $(\alpha,\calG)$-\name~private. }
 \begin{proof} The proof follows from Prop. \ref{prop:1}. Having computed the sensitivity of the reference permutation $\sigma_0$, $\Delta$, and set $\theta = \alpha / \Delta$, we are guaranteed by Property \ref{prop:1} that shuffling according to the permutation $\hat{\sigma}$ guarantees $(\alpha, \calG)$-\name privacy. 
 
%  In the algorithm, we permute by $\sigma_0^{-1} \hat{\sigma}$. Since this is equivalent to first permuting by $\hat{\sigma}$ and then permuting by $\sigma_0^{-1}$, this too guarantees $(\alpha, \calG)$-\name privacy by the immunity to post-processing property (Thm. \ref{theorem:post}).
 \end{proof}.

\subsection{Proof of Thm. \ref{thm:generalized:privacy}} \label{app:thm:generalized}

\textbf{Theorem \ref{thm:generalized:privacy}} \emph{
Alg. 1 satisfies $(\alpha',\calG')$-\name privacy for any group assignment $\calG'$ where  $ \alpha'=\alpha\frac{\Delta(\sigma_0 : \textswab{d}, \calG')}{\Delta(\sigma_0 : \textswab{d}, \calG)}$ 
}
\begin{proof}
Recall from Property \ref{prop:1} that we satisfy $(\alpha, \calG)$ \name-privacy by setting $\theta = \alpha / \Delta(\sigma_0:\textswab{d}, \calG)$. Given alternative grouping $\calG'$ with sensitivity $\Delta(\sigma_0:\textswab{d}, \calG')$, this same mechanism provides 
\begin{align*}
    \alpha' 
    &= \frac{\theta}{\Delta(\sigma_0:\textswab{d}, \calG')} \\
    &= \frac{\alpha / \Delta(\sigma_0:\textswab{d}, \calG)}{\Delta(\sigma_0:\textswab{d}, \calG')} \\
    &= \alpha \frac{\Delta(\sigma_0:\textswab{d}, \calG')}{\Delta(\sigma_0:\textswab{d}, \calG)}
\end{align*}
\end{proof}

\subsection{Formal Utility Analysis of Alg. 1}\label{app:utility:formal}
\begin{thm}For a given set $S \subset [n]$ and Hamming distance metric,  $\textswab{d}_H(\cdot)$,   Alg. 1 is $(\eta,\delta)$-preserving for $\delta=\frac{1}{\psi(\theta, \textswab{d}_H)}\sum_{h=2k+1}^{n} (e^{-\theta\cdot h} \cdot c_h)$ where \scalebox{0.9}{$k=\lceil(1-\eta)\cdot |S|\rceil$} and $c_h$ is the number of permutations with hamming distance $h$ from the reference permutation that do not preserve \scalebox{0.9}{$\eta\%$} of $S$ and is given by
\begin{gather*}c_h=\sum_{j=k+1}^{\max(l_s,\lfloor h/2\rfloor)}\binom{l_s}{j}\cdot \binom{n-l_s}{j}\cdot \Bigg[\sum_{i=0}^{\min(l_s-j,h-2j)}\binom{l_s-j}{i}\\\cdot \binom{i+j}{j}\cdot f(i,j)\cdot\binom{n-l_s-j}{h-2j-i} \cdot f(h-2j-i,j)!\Bigg]\end{gather*}\begin{gather*}
f(i,0)=!i, f(0,q)=q!\\
f(i,j)=\sum_{q=0}^{\min(i,j)}\Bigg[\binom{i}{q}\cdot\binom{j}{j-q}\cdot j! \cdot f(i-q,q)\Bigg]\\l_s=|S|, k=(1-\eta)\cdot l_s, !n=\lfloor \frac{n!}{e}+\frac{1}{2}\rfloor\end{gather*}\end{thm}
\begin{proof} Let $l_s=|S|$ denote the size of the set $S$ and $k=\lceil (1-\eta)\cdot l_S\rceil$ denote the maximum number of correct values that can be missing from $S$. Now, for a given permutation $\sigma \in \mathrm{S}_n$, let $h$ denote its Hamming distance from the reference permutation $\sigma_0$, i.e, $h=\textswab{d}_H(\sigma,\sigma_0)$. This means that $\sigma$ and $\sigma_0$ differ in $h$ indices.  Now, $h$ can be analysed in the  the following two cases, 
\par
\textbf{Case I. $h\leq 2k+1$}

For $(1-\eta)$ fraction of indices to be removed from $S$, we need at least $k+1$ indices from $S$ to be replaced by $k+1$ values from outside $S$. This is clearly not possible for $h\leq 2k+1$. Hence, here $c_h=0$. 
\par
\textbf{Case II. $h > 2k$}

For the following analysis we consider we treat the permutations as strings (multi-digit numbers are treated as a single string character). Now,   
 Let $\mathbb{S}_{\sigma_0}$ denote the non-contiguous substring of $\sigma_0$ such that it consists of all the  elements of $S$, i.e., \begin{gather}|\mathbb{S}|=l_S\\
 \forall i \in [l_S], \mathbb{S}_{\sigma_0}[i] \in S \end{gather}  Let $\mathbb{S}_{\sigma}$ denote the substring corresponding to the positions occupied by $\mathbb{S}_{\sigma_0}$ in $\sigma$. Formally,
 \begin{gather}|\mathbb{S}_{\sigma}|=l_S\\
 \forall i \in [l_S], \mathbb{S}_{\sigma_0}[i] = \sigma(\sigma_0^{-1}(\mathbb{S}_{\sigma_0}[i])) \end{gather}  For example, for $\sigma_0=(1\:2\:3\:5\:4\:7\:8\:10\:9\:6), \sigma=(1\:3\:2\:7\:8\:5\:4\:6\:10\:9)$ and $S=\{2,4,5,8\}$, we have $\mathbb{S}_{\sigma_0}=2548$ and $S_{\sigma}=3784$ where $h=\textswab{d}_H(\sigma,\sigma_0)=9$.  Let $\{\mathbb{S}_{\sigma}\}$ denote the set of the elements of string $\mathbb{S}_{\sigma}$.
 Let $A$ be the set of characters in  $\mathbb{S}_{\sigma}$ such that they do not belong to $S$, i.e, $A=\{\mathbb{S}_{\sigma}[i]|\mathbb{S}_{\sigma}[i] \not \in S, i \in [l_S]\}$. Let $B$ be the set of characters in $\mathbb{S}_{\sigma}$ that belong to $S$ but differ from $\mathbb{S}_{\sigma_0}$ in position, i.e., $B=\{\mathbb{S}_{\sigma}[i]|\mathbb{S}_{\sigma}[i] \in S, \mathbb{S}_{\sigma}[i]\neq \mathbb{S}_{\sigma_0}[i], i \in [l_S]\}$. Additionally, let $C=S-\{\mathbb{S}_{\sigma}\}$. For instance, in the above example, $A=\{3,7\}, B=\{4,8\}, C=\{2,5\}$. Now consider  an initial arrangement of $p+m$ distinct objects that are subdivided into two types -- $p$ objects of Type A and m objects of Type B. Let $f(p,m)$ denote the number of permutations of these $p+m$ objects such that the $m$ Type B objects can occupy any position but no object of Type A can occupy its original position. For example, for $f(p,0)$ this becomes the number of derangements \cite{derangement} denoted as $!p=\lfloor \frac{p!}{e} +\frac{1}{2} \rfloor$. Therefore, $f(|B|,|A|)$ denotes the number of permutations of $\mathbb{S}_\sigma$ such that $\textswab{d}_H(\mathbb{S}_{\sigma_0},\mathbb{S}_{\sigma})=|A|+|B|$. This is because if elements of $B$ are allowed to occupy their original position then this will reduce the Hamming distance.  
 \par Now, let $\bar{\mathbb{S}}_{\sigma}$ ($\bar{\mathbb{S}}_{\sigma_0}$) denote the substring left out after extracting from $\mathbb{S}_{\sigma}$ ($\mathbb{S}_{\sigma_0}$) from $\sigma$ ($\sigma_0$). For example, $\bar{\mathbb{S}}_{\sigma}=1256 10 9$ and $\bar{\mathbb{S}}_{\sigma_0}=13710 9 6$ in the above example. Let $D$ be the set of elements outside of $S$  and $A$ that occupy different positions in $\bar{\mathbb{S}}_\sigma$ and $\bar{\mathbb{S}}_{\sigma_0}$ (thereby contributing to the hamming distance), i.e., $D=\{\bar{\mathbb{S}}_{\sigma_0[i]}|\bar{\mathbb{S}}_{\sigma_0[i]} \not \in S, \bar{\mathbb{S}}_{\sigma_0[i]} \neq \bar{\mathbb{S}}_{\sigma[i]}, i \in [n-l_S]\}$. For instance, in the above example $D=\{9,6,10\}$. Hence, $h=\textswab{d}_{H}(\sigma,\sigma_0)=|A|+|B|+|C|+|D|$ and clearly $f(|D|,|C|)$ represents the number of permutations of $\bar{\mathbb{S}}_{\sigma}$ such that $\textswab{d}_H(\bar{\mathbb{S}}_{\sigma},\bar{\mathbb{S}}_{\sigma_0})=|C|+|D|$. Finally, we have 
\begin{gather*}c_h=\sum_{j=k+1}^{\max(l_s,\lfloor h/2\rfloor)}\underbrace{\binom{l_s}{j}}_{\text{\# ways of selecting set $C$}}\cdot \underbrace{\binom{n-l_s}{j}}_{\text{\# ways of selecting set $A$}}\cdot \Bigg[\\\sum_{i=0}^{\min(l_s-j,h-2j)}\underbrace{\binom{l_s-j}{i}}_{\text{\# ways of selecting set $B$}}\cdot f(i,j)\\\cdot\underbrace{\binom{n-l_s-j}{h-2j-i}}_{\text{\# ways of selecting set $D$}} \cdot f(h-2j-i,j)\Bigg]\end{gather*}
Now, for $f(i,j)$ let $E$ be the set of original positions of Type A that are occupied by Type B objects in the resulting permutation. Additionally, let $F$ be the set of the original positions of Type B objects that are still occupied by some Type B object. Clearly, Type B objects can occupy these $|E
|+|F|=m$ in any way they like. However, the type A objects can only result in $f(p-q,q)$ permutations. Therefore, $f(p,m)$ is given by the following recursive function \begin{gather*}
f(p,0)=!p\\
f(0,m)=m!\\
f(p,m)=\sum_{q=0}^{\min{p,m}}\Bigg(\underbrace{\binom{p}{q}}_{\text{\# ways of selecting set $E$}}\cdot\underbrace{\binom{m}{m-q}}_{\text{\# ways of selecting set $F$}}\\\cdot m! \cdot f(p-q,q)\Bigg)\end{gather*}

Thus, the total probability of failure is given by 
\begin{gather}\delta=\frac{1}{\psi(\theta, \textswab{d}_H)}\sum_{h=2k+2}^{n} (e^{-\theta\cdot h} \cdot c_h)\end{gather}
\end{proof}

\newpage
\subsection{Additional Experimental Details}\label{app:extraresults}
\subsubsection{Evaluation of $(\eta,\delta)$-preservation}\label{app:numerical}

\begin{figure*}[ht]
    \begin{subfigure}[b]{0.33\linewidth}\centering
    \includegraphics[width=\linewidth]{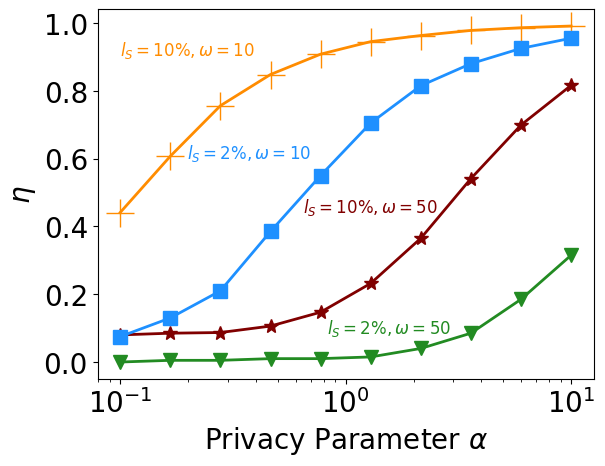}
        \caption{Variation with $\alpha$}
        \label{fig:eta:alpha}
    \end{subfigure}
    \begin{subfigure}[b]{0.33\linewidth}\centering
    \includegraphics[width=\linewidth]{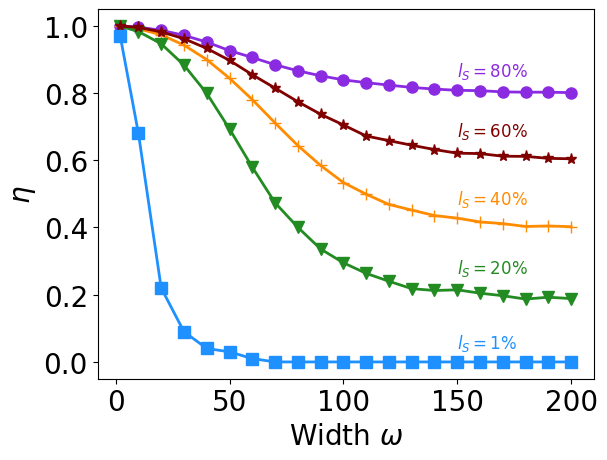}
        \caption{Variation with $\omega$; $\alpha = 3$}
        \label{fig:eta:width}
    \end{subfigure}
    \begin{subfigure}[b]{0.33\linewidth}\centering
    \includegraphics[width=\linewidth]{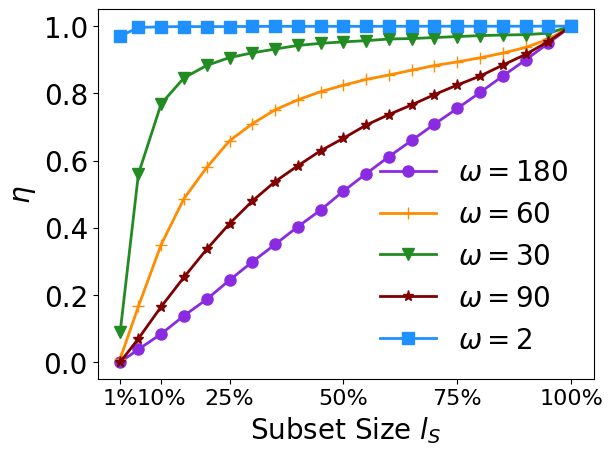}
        \caption{Variation with $l_S$; $\alpha = 3$}
        \label{fig:eta:subset}
    \end{subfigure}
        \caption{$(\eta,\delta)$-Preservation Analysis}
        \label{fig: eta delta preservation}
\end{figure*}

In this section, we evaluate the characteristics of the  $(\eta,\delta)$-preservation for Kendall's $\tau$ distance $\textswab{d}_\tau(\cdot, \cdot)$.

Each sweep of Fig. \ref{fig: eta delta preservation} fixes $\delta = 0.01$, and observes $\eta$. We consider a dataset of size $n = 10K$ and a subset $S$ of size $l_S$ corresponding to the indices in the middle of the reference permutation $\sigma_0$ (the actual value of the reference permutation is not significant for measuring preservation). For the rest of the discussion, we denote the width of a permutation by $\omega$ for notational brevity. For each value of the independent axis, we generate $50$ trials of the permutation $\sigma$ from a Mallows model with the appropriate $\theta$ (given the $\omega$ and $\alpha$ parameters). We then report the largest $\eta$ (fraction of subset preserved) that at least 99\% of trials satisfy. 

In Fig. \ref{fig:eta:alpha}, we see that preservation is highest for higher $\alpha$ and increases gradually with declining width $\omega$ and increasing subset size $l_s$. 

Fig. \ref{fig:eta:width} demonstrates that preservation declines with increasing width. $\Delta$ increases quadratically with width $\omega$ for $\textswab{d}_\tau$, resulting in declining $\theta$ and increasing randomness. We also see that larger subset sizes result in a more gradual decline in $\eta$. This is due to the fact that the worst-case preservation (uniform random shuffling) is better for larger subsets. i.e. we cannot do worse than $80\%$ preservation for a subset that is $80\%$ of indices. 

Finally, Fig. \ref{fig:eta:subset} demonstrates how preservation grows rapidly with increasing subset size. For large widths, we are nearly uniformly randomly permuting, so preservation will equal the size of the subset relative to the dataset size. For smaller widths, we see that preservation offers diminishing returns as we grow subset size past some critical $l_s$. For $\omega = 30$, we see that subset sizes much larger than a quarter of the dataset gain little in preservation. 

\subsubsection{Adult Dataset}
\label{app:adult experiments}
\begin{figure*}[ht]
    \begin{subfigure}[b]{0.33\linewidth}\centering
        \includegraphics[width=\linewidth]{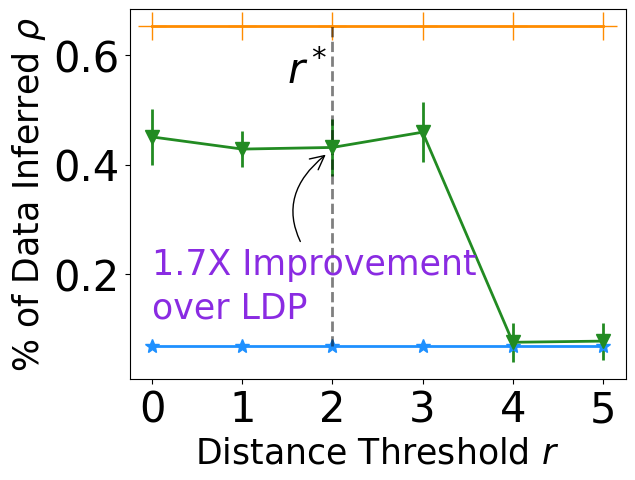}
        %\vspace{-0.15cm}
        \caption{\textit{Adult}: Attack }%($r$)}
        \label{fig:Adult:attack}
    \end{subfigure}
    \begin{subfigure}[b]{0.33\linewidth}\centering
    \includegraphics[width=\linewidth]{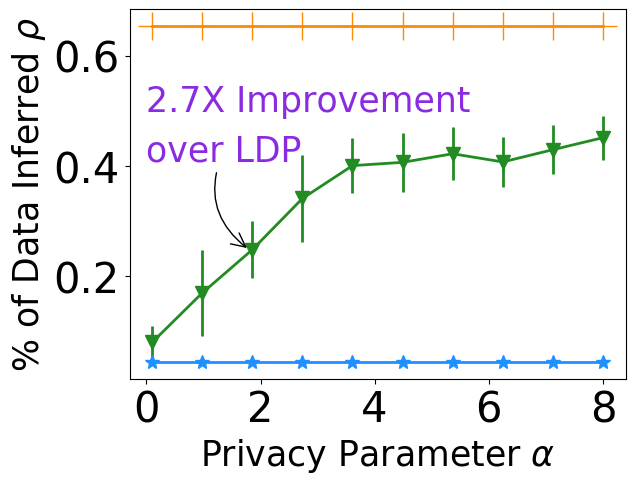}
        %   \vspace{-0.15cm} 
        \caption{\textit{Adult}: Attack ($\alpha$)}
        \label{fig:Adult:alpha}
    \end{subfigure}
    \begin{subfigure}[b]{0.33\linewidth}\centering
   \includegraphics[width=\linewidth]{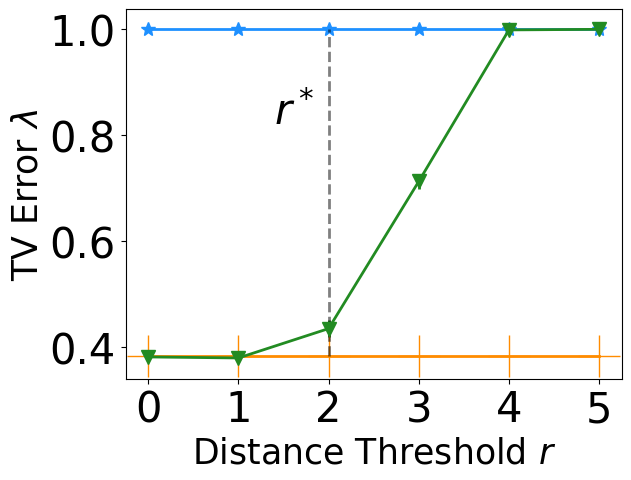}
        %\vspace{-0.15cm}
        \caption{\textit{Adult}: Learnability}
        \label{fig:Adult:utility}
    \end{subfigure}
        \caption{Adult dataset experiments}
        \label{fig: adult experiments}
\end{figure*}

\subsection{Additional Related Work}\label{app:related}
In this section, we discuss the relevant existing work. \par  The anonymization of noisy responses  to improve differential privacy was first proposed by Bittau et al. \cite{Bittau2017} who proposed a principled system architecture for shuffling. This model was formally studied later in \cite{shuffling1, shuffle2}. Erlingsson et al. \cite{shuffling1} showed that for arbitrary $\epsilon$-\ldp randomizers, random shuffling results in privacy amplification. Cheu et al. \cite{shuffle2} formally defined the shuffle \DP model
and analyzed the privacy guarantees of the binary randomized response in this model.
The shuffle \DP model differs from our approach in  two ways. First, it focuses completely on the \DP guarantee. %\ldp is characterized by a privacy parameter (see Section \ref{sec:background}) $\epsilon$, lower the value of $\epsilon$ stronger is the guarantee achieved. 
The privacy amplification is manifested in the from of a lower $\epsilon$ (roughly a factor of $\sqrt{n}$) when viewed in an alternative \DP model known as the central \DP model. \cite{shuffling1,shuffle2,blanket,feldman2020hiding,Bittau2017,Balcer2020SeparatingL}.  
However, our result caters to local inferential privacy. Second, the shuffle model involves an uniform random shuffling of the entire dataset. In contrast, our approach the granularity at which the data is shuffled is tunable which delineates a threshold for the learnability of the data. 
\par A steady line of work has sudied the inferential privacy setting \cite{semantics, Kifer,  IP, Dalenius:1977, dwork2010on, sok}. Kifer et al. \cite{Kifer} formally studied privacy degradation in the face of data correlations and later proposed a  privacy framework, Pufferfish \cite{Pufferfish, Song,Blowfish}, for analyzing inferential privacy. Subsequently, several other privacy definitions have also been proposed for the inferential privacy setting \cite{DDP,BDP,correlated,correlated2,CWP}. For instance, Gehrke et al.  proposed a zero-knowledge privacy \cite{ZKPrivacy,crowd} which is based on simulation semantics. Bhaskar et al. proposed  noiseless privacy \cite{noiseless, TNP} by restricting the set of prior
distributions that the adversary may have access to.  A recent work by Zhang et al. proposes attribute privacy \cite{AP} which focuses on the sensitive properties of a whole dataset. In another recent work, Ligett et al. study a relaxation of \DP that accounts for mechanisms that leak
some additional, bounded information about the database 
\cite{bounded}. Some early work in local inferential privacy include profile-based privacy \cite{profile} by Gehmke et al. where the problem setting comes with a graph of data generating distributions, whose edges encode sensitive pairs of distributions that should be made indistinguishable. In another work by Kawamoto et al., the authors propose distribution privacy \cite{takao} -- local differential privacy for probability distributions.    The major difference between our work and prior research is that we provide local inferential privacy through a new angle -- data shuffling. 

Finally, older works such as $k$-anonymity \cite{kanon},  $l$-diversity \cite{ldiv}, and Anatomy \cite{anatomy} and other \cite{older1, older2, older3, older4, older5} have studied the privacy risk of non-sensitive auxiliary information, or `quasi identifiers' (QIs). In practice, these works focus on the setting of dataset release, where we focus on dataset collection. As such, QIs can be manipulated and controlled, whereas we place no restriction on the amount or type of auxiliary information accessible to the adversary, nor do we control it. Additionally, our work offers each individual formal inferential guarantees against informed adversaries, whereas those works do not. We emphasize this last point since formalized guarantees are critical for providing meaningful privacy definitions. As established by Kifer and Lin in \emph{An Axiomatic View of Statistical Privacy and Utility} (2012), privacy definitions ought to at least satisfy post-processing and convexity properties which our formal definition does.

%Relations between differential privacy andthat has given rise to a set of privacy definitions \cite{DDP,BDP,correlated, correlated2, CWP} (not an exhaustive list)  including zero-knowledge privacy \cite{ZKPrivacy}, crowd-blending privacy \cite{crowd}, noiseless privacy \cite{noiseless, TNP}, Pufferfish framework \cite{Pufferfish, Blowfish, Song} and attribute privacy \cite{AP}. 

\end{document}